\documentclass{article}

\usepackage{microtype}
\usepackage{graphicx}
\usepackage{subcaption}
\usepackage{booktabs} 
\usepackage{multirow}

\usepackage{hyperref}

\usepackage[preprint]{icml2026}

\usepackage{amsmath}
\usepackage{amssymb}
\usepackage{mathtools}
\usepackage{amsthm}
\usepackage{pifont}

\usepackage{bm}

\usepackage[capitalize,noabbrev]{cleveref}

\theoremstyle{plain}
\newtheorem{theorem}{Theorem}[section]

\newtheorem{assumption}[theorem]{Assumption}
\newtheorem{remark}[theorem]{Remark}

\usepackage[utf8]{inputenc}
\usepackage{xcolor}

\makeatletter

\expandafter\let\csname algorithm*\endcsname\relax
\expandafter\let\csname endalgorithm*\endcsname\relax

\makeatother

\usepackage[linesnumbered,ruled,vlined]{algorithm2e}

\definecolor{commentcolor}{RGB}{0, 128, 128}
\newcommand{\mycomment}[1]{\textcolor{commentcolor}{\# #1}}
\usepackage{colortbl}  
\usepackage{xcolor}    
\usepackage{arydshln}  

\definecolor{gray20}{gray}{0.90} 
\definecolor{gray10}{gray}{0.95} 
\definecolor{highlightcyan}{HTML}{D7F6FF} 

\newcolumntype{I}{!{\vrule width 1.2pt}}

\makeatletter
\newcommand{\thickhline}{%
    \noalign {\ifnum 0=`}\fi \hrule height 1.2pt
    \futurelet \reserved@a \@xhline
}
\makeatother

\usepackage[textsize=tiny]{todonotes}

\icmltitlerunning{Submission and Formatting Instructions for ICML 2026}

\begin{document}

\twocolumn[
  \icmltitle{UTOPIA: Unlearnable Tabular Data via Decoupled Shortcut Embedding}



  \icmlsetsymbol{equal}{*}

  \begin{icmlauthorlist}
    \icmlauthor{Jiaming He}{1}
    \icmlauthor{Fuming Luo}{2}
    \icmlauthor{Hongwei Li}{1}
    \icmlauthor{Wenbo Jiang}{1}
    \icmlauthor{Wenshu Fan}{1}
    \icmlauthor{Zhenbo Shi}{3}
    \icmlauthor{Xudong Jiang}{4}
    \icmlauthor{Yi Yu}{4}
  \end{icmlauthorlist}

  \icmlaffiliation{1}{University of Electronic Science and Technology of China}
  \icmlaffiliation{2}{Independent Researcher}
  \icmlaffiliation{3}{University of Science and Technology of China}
  \icmlaffiliation{4}{Nanyang Technological University}


  \icmlkeywords{Machine Learning, ICML}

  \vskip 0.3in
]



\printAffiliationsAndNotice{}  

\begin{abstract}

Unlearnable examples (UE) have emerged as a practical mechanism to prevent unauthorized model training on private vision data, while extending this protection to tabular data is non-trivial.
Tabular data in finance and healthcare is highly sensitive, yet existing UE methods transfer poorly because tabular features mix numerical and categorical constraints and exhibit saliency sparsity, with learning dominated by a few dimensions. Under a Spectral Dominance condition, we show certified unlearnability is feasible when the poison spectrum overwhelms the clean semantic spectrum. Guided by this, we propose \textbf{\underline{U}}nlearnable \textbf{\underline{T}}abular Data via Dec\textbf{\underline{O}}u\textbf{\underline{P}}led Shortcut Embedd\textbf{\underline{I}}ng (\textbf{UTOPIA}), which exploits feature redundancy to decouple optimization into two channels: high saliency features for semantic obfuscation and low saliency redundant features for embedding a hyper correlated shortcut, yielding constraint-aware dominant shortcuts while preserving tabular validity. Extensive experiments across tabular datasets and models show UTOPIA drives unauthorized training toward near random performance, outperforming strong UE baselines and transferring well across architectures.
\end{abstract}

\section{Introduction}
Tabular data, structured into rows of heterogeneous numerical and categorical features \cite{tabular_survey_2, tabular_survey_5}, stands as the predominant data modality for machine learning applications across critical domains, \textit{e.g.,} finance~\cite{sattarov2023findiff}, healthcare~\cite{hernandez2022synthetic}, and scientific research. Unlike unstructured data such as images and text, tabular data relies on heterogeneous structural constraints to ensure semantic validity, which integrates continuous numerical with discrete categorical features to capture complex data dependencies \cite{tabular_survey_1, tabular_survey_4}.

As deep learning models continue to scale, they increasingly rely on massive data harvested from the public web. As a countermeasure, unlearnable examples (UE)~\cite{unlearnable_data_survey_1} is introduced to embed imperceptible perturbations (as a ``shortcut") in private data that mislead unauthorized models into learning noise rather than true semantics. As the most common format for public records, tabular data is uniquely vulnerable to mass harvesting. Its compact structure allows adversaries to effortlessly scrape high-value assets from the open web, demonstrating an urgent need to defend tabular data from unauthorized training. As illustrated in Figure~\ref{fig:teaser}, the data owner releases a protected tabular dataset, ensuring that unauthorized model training yields poor performance.

\begin{figure}[t]
    \centering
    \includegraphics[width=1\linewidth]{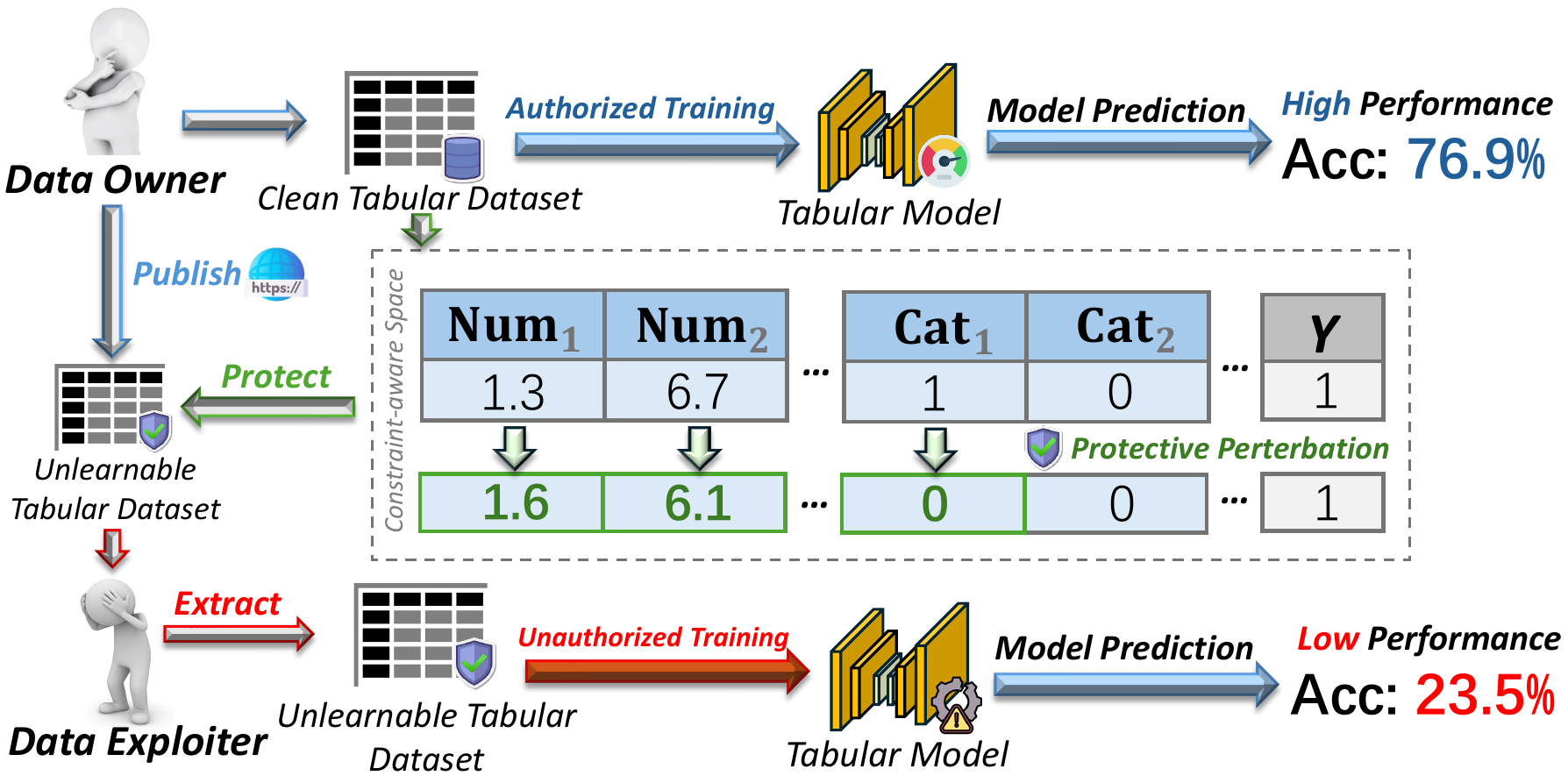}
    \vspace{-6mm}
    \caption{Illustration of the UTD that prevents unauthorized training on tabular data by injecting constraint-aware perturbations, transforming clean records into an unlearnable dataset that leads to poor generalization and performance on clean test data.}
    \label{fig:teaser}
    \vspace{-4mm}
\end{figure}

While existing UE methods~\cite{em,tap,rem,sep,wang2025provably} have successfully safeguarded image datasets by optimizing perturbations over \textit{continuous data distributions}, generating them for the tabular domain remains a formidable challenge.
Unlike homogeneous pixel arrays, a tabular sample $\mathbf{x}$ resides in a \textit{discrete and heterogeneous feature space}, comprising a mix of continuous numerical features $\mathbf{x}_{\text{num}} \in \mathcal{X}_{\text{num}}$ and discrete categorical features $\mathbf{x}_{\text{cat}} \in \mathcal{X}_{\text{cat}}$. 
Crucially, this discrete modality is governed by rigorous integrity constraints that fundamentally differ from the continuous manifolds assumed by existing methods. 
Specifically, any protective perturbations must ensure that the modified sample satisfies structural validity, formally denoted as $(\mathbf{x}+\boldsymbol{\delta}) \in \mathcal{C}_{\text{struc}}$. 
Thus, directly applying existing global-based perturbation to this discrete space is infeasible, where samples must satisfy structural validity $\mathcal{C}_{\text{struc}}$ to maintain data integrity.

In this work, we introduce the concept of Unlearnable Tabular Data (\textbf{UTD}). We first analyze the uniqueness of tabular data learning to reveal our motivations. We discover that tabular models typically exhibit \textit{data representation sparsity}, relying on a subset of high-influence features while leaving a vast subspace of redundant features dormant. This renders standard noise injection ineffective, as models easily filter out perturbations on non-salient dimensions. Our theoretical insight is to weaponize this \textit{intrinsic feature sparsity} to enforce a \textit{spectral dominance} condition ($\kappa \!=\! \lambda_p / \lambda_c \!\gg\! 1$), where the spectral norm of the injected poison mathematically overwhelms the clean semantic signal.

Building on these insights, we propose a unified UE method for tabular data, termed \textbf{\underline{U}}nlearnable \textbf{\underline{T}}abular Data via Dec\textbf{\underline{O}}u\textbf{\underline{P}}led Shortcut Embedd\textbf{\underline{I}}ng (\textbf{UTOPIA}) to empirically instantiate the UTD. Unlike image-based UE methods that inject external patterns, UTOPIA fundamentally alters the feature topology by decoupling the feature space into conflicting optimization trajectories. It suppresses semantic learnability on high-saliency features (minimizing $\lambda_c$) while simultaneously amplifying a synthesized shortcut on redundant ones (maximizing $\lambda_p$). Additionally, to accommodate the discrete nature of tabular data, we employ a constraint-aware perturbation strategy without violating data integrity. 

\textbf{Our contributions} can be summarized as below:

$\bullet$ To the best of our knowledge, we are the first to formally define the problem of UTD, identifying the unique challenges of manifold heterogeneity and saliency sparsity that render existing image-based UE ineffective.

$\bullet$ We derive a rigorous theoretical framework proving that creating a dominant spectral shortcut ($\kappa \gg 1$) is a necessary and sufficient condition for certified unlearnability.

$\bullet$ We propose UTOPIA, guided by the observation of data representation sparsity and the theoretical guarantee that spectral dominance ($\kappa \gg 1$) ensures unlearnability, which synthesizes robust shortcuts by decoupling features into conflicting optimization trajectories.

$\bullet$ Extensive experiments demonstrate that UTOPIA significantly outperforms state-of-the-art baselines, successfully reducing unauthorized model accuracy to random guessing level across diverse datasets and model architectures.

\section{Related Work}
\subsection{Data Poisoning and Unlearnable Examples}
To protect data from unauthorized model training, research has moved beyond traditional data poisoning. While poisoning attacks \cite{barreno2010security, goldblum2022dataset}, categorized as either integrity (\textit{e.g.,} backdoor attacks \cite{gu2017badnets,schwarzschild2021just}) or availability attacks \cite{biggio2012poisoning,xiao2015feature}, rely on injecting conspicuous, unrestricted samples \cite{koh2017understanding,zhao2022clpa,10.5555/3618408.3619357}, a more stealthy paradigm has emerged: unlearnable examples (UE).

UE aims to render an entire dataset non-exploitable by adding imperceptible perturbations (\textit{e.g.,} $\Vert \boldsymbol{\delta} \Vert_{\infty} \!\le\! \frac{8}{255}$) that preserve the original labels \cite{em,zhang2023unlearnable,DBLP:conf/icml/ZhuMDG24,lin2024safeguarding,meng2024semantic,yu2024purify,yu2024unlearnable,wu2025temporal}. The core mechanism is to induce models to learn these perturbation patterns instead of genuine semantic features, leading to near-random performance on clean test data. The evolution of UE methods began with error-minimizing noise (EM) \cite{em} and noise generation via neural tangent kernels (NTGA) \cite{ntga}. It progressed to enhance robustness against countermeasures, such as using targeted adversarial examples (TAP) \cite{tap} or crafting UE resistant to adversarial training (REM) \cite{rem}. Recent innovations focus on surrogate-free generation (LSP, AR) \cite{lsp,ar} and improving robustness to augmentations via techniques like one-pixel shortcuts (OPS) \cite{ops}. The field is also maturing theoretically, with works on Provably UE offering formal robustness certificates \cite{wang2025provably}.

UEs is now expanding beyond standard image classification. In 3D vision, \citet{DBLP:conf/icml/ZhuMDG24,wang2024unlearnable3d} explore UE for point clouds. For dense prediction tasks, universal generators like UnSeg \cite{NEURIPS2024_90672770} protect data against segmentation models. The paradigm has also been adapted to multi-task learning \cite{yu2025mtl}. However, a critical gap exists in applying UE to tabular data, a high-stakes domain in finance and healthcare. The unique structure of tabular data, with its mix of categorical and numerical features, demands fundamentally new perturbation strategies, making it a vital open research direction.

\subsection{Tabular Data Learning}
Tabular data, characterized by heterogeneous features and a lack of spatial correlation, was dominated by tree-based ensemble methods. Gradient Boosted Decision Trees (GBDTs), such as XGBoost \cite{chen2016xgboost} and CatBoost \cite{prokhorenkova2018catboost}, have long been considered the state-of-the-art due to their efficiency and robustness on irregular decision boundaries \cite{grinsztajn2022why}. However, recent years have witnessed a surge in deep learning for tabular data (TDL), driven by the need to integrate tabular inputs into differentiable pipelines. Advanced architectures like FT-Transformer \cite{gorishniy2021revisiting} have successfully adapted the Transformer architecture to tabular structures by tokenizing features, achieving performance competitive with GBDTs.

Beyond standard transformers, ExcelFormer \cite{ExcelFormer} utilizes an alternating attention mechanism and specialized data preprocessing to better capture complex feature interactions \cite{tablular_survey_7_Transformers}. To enhance model stability, TROMPT \cite{Trompt} integrates Ordinary Differential Equations (ODEs) with prompt-based learning to handle irregular patterns in tabular features. More recently, TabR \cite{TabR} introduced a retrieval-augmented approach that improves predictions by attending to similar examples from the training set.


\section{Preliminary and Problem Formulation}
\label{sec:problem_formulation}
In this section, we formally characterize the problem of generating unlearnable examples in tabular domains. We model the interaction between the data protector and the data exploiter (unauthorized model trainer) as a bi-level optimizations, constrained by the heterogeneous property of the feature space and tabular data structure integrity.

\textbf{Threat Model.}
We consider a threat model~\cite{em,wang2025provably,yu2025mtl} involving a \textit{data owner} and a \textit{data exploiter (adversary)}. The \textit{data owner} seeks to release a protected (unlearnable) tabular dataset to be posted public $\mathcal{D}_u = \{(\mathbf{x}_i \oplus \boldsymbol{\delta}_i, y_i)\}_{i=1}^N$ by applying an imperceptible, protective perturbation $\boldsymbol{\delta}_i$. The goal is to render $\mathcal{D}_u$ unexploitable such that any model $f_\theta: \mathcal{X} \to \mathcal{Y}$ trained on $\mathcal{D}_u$ fails to generalize to the clean data distribution $\mathcal{P}$. Conversely, the \textit{data exploiter} aims to train parameters $\theta$ on $\mathcal{D}_u$ that minimize the empirical risk $\mathcal{L}$, having no access to $\mathcal{D}_c$ but full control over the protected dataset $\mathcal{D}_u$ and the model learning process.

Let $\mathcal{D}_c = \{(\mathbf{x}_i, y_i)\}_{i=1}^{N}$ denote a clean source dataset drawn i.i.d. from an underlying distribution $\mathcal{P}$ over $\mathcal{X} \times \mathcal{Y}$, where $\mathcal{X}$ is the input feature space and $\mathcal{Y} = \{1, \dots, K\}$ is the label space. Unlike homogeneous domains (\textit{e.g.,} image), the tabular feature space $\mathcal{X}$ is a Cartesian product of heterogeneous subspaces: $\mathcal{X} = \mathcal{X}_{\text{num}} \times \mathcal{X}_{\text{cat}}$, where $\mathcal{X}_{\text{num}} \subseteq \mathbb{R}^{d_{\text{num}}}$ and $\mathcal{X}_{\text{cat}} \subseteq \mathbb{Z}^{d_{\text{cat}}}$ correspond to continuous numerical and discrete categorical features, respectively. Consequently, any sample $\mathbf{x} \in \mathcal{X}$ decomposes into $\mathbf{x} = [\mathbf{x}_{\text{num}}, \mathbf{x}_{\text{cat}}]$.

\textbf{Constraint-aware Feasible Space.}
A defining characteristic of tabular data is the presence of rigorous data integrity constraints. Blind application of gradient-based perturbations often yields samples that violate domain semantics. Following \citet{ben2024cafa}, we define the feasible perturbation set $\Omega(\mathbf{x})$ as the intersection of structural validity.
Let $\mathcal{C}_{\text{struc}} \subset \mathcal{X}$ denote the set of structurally valid samples, ensuring numerical values remain within bounded domains and categorical indices map to valid tokens. The feasible perturbation space for a given input $\mathbf{x}$ is defined as:
\begin{equation}
    \Omega(\mathbf{x}) = \left\{ \boldsymbol{\delta} \mid \|\boldsymbol{\delta}\|_p \leq \epsilon,~(\mathbf{x} \oplus \boldsymbol{\delta}) \in \mathcal{C}_{\text{struc}} \right\},
\end{equation}
where $\|\cdot\|_p$ denotes a norm suitable for mixed-type data (\textit{e.g.,} standardized $\ell_\infty$ for numerical and Hamming distance for categorical components) and $\epsilon$ is the budget.

\textbf{Optimization Objective.}
Formally, we frame the synthesis of unlearnable examples as a bi-level optimization game between the defender (data protector) and the adversary (data exploiter). The defender's objective is to craft a perturbation $\boldsymbol{\delta}$ within the feasible tabular space $\Omega(\mathbf{x})$ that maximizes the generalization error (risk) on the clean distribution $\mathcal{P}$, subject to the adversary's optimal response on the protected data. The goal of protective perturbation is formulated as:
\begin{equation}
\label{eq:bilevel_problem}
\begin{aligned}
    & \max_{\boldsymbol{\delta} \in \Omega(\mathbf{x})} \quad \mathbb{E}_{(\mathbf{x}, y) \sim \mathcal{P}} \left[ \mathcal{L}(f_{\theta^{*}(\boldsymbol{\delta})}(\mathbf{x}), y) \right] \\
    & \text{s.t.} \quad \theta^{*}(\boldsymbol{\delta}) \in \arg\min_{\theta} \sum_{(\mathbf{x}_i, y_i) \in \mathcal{D}_c} \mathcal{L}(f_\theta(\mathbf{x}_i \oplus \boldsymbol{\delta}_i), y_i).
\end{aligned}
\end{equation}
The adversary trains the model by minimizing the loss on $\mathcal{D}_u$, thereby learning the injected spurious shortcuts. The outer maximization ensures these shortcuts are statistically orthogonal to the true semantics, causing the learned parameters $\theta^{*}$ to fail to generalize to the clean domain $\mathcal{P}$.

\section{Methodology}
\subsection{Theoretical Analysis: Certified Unlearnability with Spectral Dominance}
\label{sec:theory}
In this section, we present a rigorous theoretical framework for our method. By leveraging an explicit characterization of the gradient dynamics in the dual space, we derive a closed-form upper bound on the certified learnability of tabular data under spectral manipulation.

\begin{assumption}[Orthogonal Subspace Decomposition.]
\label{assump}
Let tabular dataset $\mathcal{D}$ be the distribution over $\mathcal{X} \times \mathcal{Y}$, where data instance $\mathcal{X} \subseteq \mathbb{R}^d$ and respective labels $\mathcal{Y} = \{-1, 1\}$. We assume that the tabular data manifold admits a decomposition into orthogonal subspaces. For any $\mathbf{x} \sim \mathcal{D}$, we have $\mathbf{x} = \mathbf{x}_c + \mathbf{x}_p$, satisfying the orthogonality condition $\langle \mathbf{x}_c, \mathbf{x}_p \rangle = 0$ almost surely. We model the feature statistics via their spectral characteristics:
\begin{equation}
\mathbf{x}_c \sim \mathcal{N}(\boldsymbol{\mu}_{y}^c, \mathbf{\Sigma}_c), 
\quad 
\mathbf{x}_p \sim \mathcal{N}(\boldsymbol{\mu}_{y}^p, \mathbf{\Sigma}_p).
\end{equation}
Here, $\boldsymbol{\mu}_{y}^c$ and $\boldsymbol{\mu}_{y}^p$ are the class-conditional means for clean subspace $\mathbf{x}_c$ and poison subspace $\mathbf{x}_p$, respectively. Let $\lambda_c \coloneqq \|\mathbf{\Sigma}_c\|_2=\sigma_{max}(\mathbf{\Sigma}_c)$ and $\lambda_p \coloneqq \|\mathbf{\Sigma}_p\|_2=\sigma_{max}(\mathbf{\Sigma}_p)$ denote the spectral norms (feature strengths) of these subspaces. We define the \emph{Spectral Imbalance Ratio} as $\kappa \coloneqq \frac{\lambda_p}{\lambda_c}$.
We denote by $\mathcal{S}_c \subseteq \mathbb{R}^d$ the clean subspace corresponding to $\mathbf{x}_c$ (\textit{i.e.,} samples with $\mathbf{x}_p = 0$). The clean dataset $\mathcal{D}_{\text{clean}}$ is from the distribution of $\mathcal{D}$ given that $\mathbf{x} \in \mathcal{S}_c$.
\end{assumption}
%

\textbf{Analytical Perspective on Learning Dynamics.}
Building upon this orthogonal decomposition in Assumption~\ref{assump}, we analyze the learning dynamics of a ridge-regularized linear classifier \(f(\mathbf{x}) = \mathbf{w}^\top \mathbf{x}\) to precisely quantify how spectral dominance suppresses the clean features. The optimization objective is given by the empirical risk:
\begin{equation}
    \min_{\mathbf{w}} \mathcal{R}(\mathbf{w}) 
    \!\coloneqq\! 
    \sum_{i=1}^n \log\bigl(1 \!+\! \exp(-y_i \mathbf{w}^\top \mathbf{x}_i)\bigr) 
    \!+\! \frac{\gamma}{2} \|\mathbf{w}\|_2^2.
\end{equation}
Direct analysis of the primal variable $\mathbf{w}$ is intractable due to the coupling induced by the logistic loss. Following \citet{pezeshki2021gradient}, we analyze the dual variables $\boldsymbol{\alpha} \in \mathbb{R}^n$ via the Legendre transformation. The fixed point of the learning dynamics can be characterized analytically using the Lambert $W$ function, denoted as $W(\cdot)$, which is the inverse function of $f(z) = z e^z$. 
This analytical setup allows us to derive a closed-form upper bound on the norm of the clean weights, formalized as follows.

\begin{theorem}[Analytical Clean Weight Suppression]
\label{thm:analytical_suppression}
Under the orthogonality Assumption~\ref{assump}, the learning dynamics decouple along the clean and poison subspaces. Let $\mathbf{w}^* = \mathbf{w}_c^* + \mathbf{w}_p^*$ denote the equilibrium solution of the gradient dynamics. Then there exists a constant $\xi > 0$ (encoding label alignment) and a strictly increasing function $\psi: (0,\infty) \to (0,\infty)$ such that the equilibrium norm of the weights associated with the clean subspace satisfies
\begin{equation}
    \|\mathbf{w}_c^*\|_2 
    \;\le\; 
    \frac{\xi}{\kappa}\, W\!\bigl(\psi(\lambda_c)\bigr),
    \label{eq:lambert_bound}
\end{equation}
where $\kappa \!=\! {\lambda_p}/{\lambda_c}$ is the spectral imbalance ratio. For any fixed $\lambda_c$ and strength $\gamma$, the clean weights $\mathbf{w}_c^*$ vanish at least at the rate $\mathcal{O}(1/\kappa)$ as $\kappa \to \infty$. Proof is in Appendix \ref{prf:analytical_suppression}.
\end{theorem}

\textbf{Certified Unlearnability Theorem.}
Building on Theorem~\ref{thm:analytical_suppression}, we derive a probabilistic certificate that upper bounds the classifier’s accuracy under parameter space perturbations, which in turn yields a certified lower bound on unlearnability.
 We consider a randomized linear classifier $\tilde{f}(\mathbf{x})=\operatorname{sign}((\mathbf{w}^*+\boldsymbol{\delta})^\top \mathbf{x})$ with $\boldsymbol{\delta}\sim\mathcal{N}(\mathbf{0},\sigma^2\mathbf{I})$, where $\mathbf{w}^*$ is the learned weight vector and $\sigma$ controls the magnitude of local weight uncertainty that can arise from training stochasticity. Our goal is to certify, for a given $\sigma$, $\mathbb{P}(\tilde{f}(\mathbf{x})=y)$ by bounding the probability that the signed margin $(\mathbf{w}^*+\boldsymbol{\delta})^\top \mathbf{x}$ preserves its label consistent sign. This connects weight suppression to an explicit accuracy guarantee.

\begin{theorem}[Certified Accuracy Bound via Lambert Dynamics]
\label{thm:certified_bound}
Given a tabular dataset with spectral imbalance ratio $\kappa$, the expected accuracy of the smoothed classifier on the clean distribution $\mathcal{D}_{\text{clean}}$ is upper-bounded by a function of the Lambert $W$ solution:
\begin{equation}
    \underset{(\mathbf{x}, y) \sim \mathcal{D}_{\text{clean}}}{\mathbb{E}} 
    \!\left[ \mathbb{P}(\tilde{f}(\mathbf{x}) = y) \right] 
    \!\le\! 
    \Phi\Big( 
        \tfrac{\mathcal{C}}{\sigma} \cdot 
        \underbrace{\tfrac{W(\psi(\lambda_c))}{\kappa}}_{\text{Suppression Term}} 
    \Big),
\end{equation}
where $\Phi(\cdot)$ is the Gaussian CDF, $\psi(\cdot)$ is the strictly increasing function from Theorem~\ref{thm:analytical_suppression}, and $\mathcal{C} > 0$ is a data-dependent constant. Proof is in Appendix \ref{prf:certified_bound}.
\end{theorem}

\begin{remark}
\textit{The introduction of the Lambert $W$ function in Theorem~\ref{thm:analytical_suppression} provides a compact parametrization of the non-linear dependence of the equilibrium weights on the clean spectrum through the response function $\mathcal{A}(\lambda,\gamma)$. The key suppression factor is the explicit $1/\kappa$ scaling established therein. Combined with the probabilistic certification framework of Theorem~\ref{thm:certified_bound}, this shows that spectral dominance ($\kappa\!=\! {\lambda_p}/{\lambda_c} \!\gg\! 1$) is a sufficient condition for certified unlearnability in tabular manifolds: as the poison spectrum overwhelms the clean spectrum, the effective clean margin collapses, and the certified accuracy cannot exceed $0.5$.
}
\end{remark}

\subsection{Unlearnable Tabular Data via Decoupled Shortcut Embedding (UTOPIA)}
\label{sec:methodology}

\textbf{\textit{What hinders current UE in the tabular domain?}}
The primary barrier stems from the necessity of performing optimization within a rigorous \textit{discrete space}. Unlike continuous image manifolds, tabular data resides in a heterogeneous space where blindly applying global perturbations inevitably violates structural integrity constraints. Furthermore, we observe a critical phenomenon of \textit{data representation sparsity}: tabular models typically rely on a distinct subset of high-influence features, rendering uniform noise injection ineffective as models easily filter out perturbations on redundant features, and we detailed this with experimental analysis in Appendix \ref{appendix:Feature_Dependency}. Therefore, we raise the question: 

\textit{Can we transcend the limitations of continuous optimization to rigorously perturb heterogeneous tabular manifolds, leveraging intrinsic feature sparsity to synthesize certified spectral shortcuts?}

Building on our observation of \textit{data representation sparsity} of tabular data and the theoretical guaranty that spectral dominance ($\kappa= \lambda_p/{\lambda_c} \gg 1$) ensures unlearnability, we introduce \textbf{Unlearnable Tabular Data via Decoupled Shortcut Embedding (UTOPIA)}, a practical framework designed to empirically achieve this condition by synthesizing protective perturbations that maximize $\kappa$ through constrained optimization. Unlike prior works that inject external fixed patterns, UTOPIA exploits the \textit{intrinsic feature sparsity} unique to tabular data. Our core insight is to synthesize a \textit{dominant spectral confounder} by decoupling the tabular features into a dominant group and a recessive group, while assigning divergent optimization trajectories. By steering these naturally entangled features in opposing directions, we compel the model to converge onto a simplified latent structure that dominates the optimization landscape while remaining statistically disjoint from the true semantic.

\textbf{Influence-Guided Subspace Decoupling.}
\label{sec:feature_decoupling}
In heterogeneous tabular domains, preserving the manifold structure is paramount. To identify the optimal substrate for embedding spurious correlations within the numerical subspace $\mathbf{x}_{\text{num}}$, we employ an \textit{Influence-Guided Decoupling} strategy. Let $\mathbf{R} \in \mathbb{R}^{d \times d}$ denote the feature correlation matrix and $\mathcal{P} = \{(i, j) \mid i < j, |\mathbf{R}_{ij}| > \tau\}$ represent the set of redundant feature pairs. We formalize the subspace decoupling as a unified operator that first quantifies the expected \textit{gradient attribution}, and subsequently partitions the redundant pairs into orthogonal index sets $\mathcal{I}_{\Phi}$ (Predictive Dominant) and $\mathcal{I}_{\Psi}$ (Predictive Recessive):
\begin{equation}
\label{eq:decoupling_mechanism}
\begin{aligned}
    \mathcal{I}_{\Phi} &\!=\! \bigcup_{(i, j) \in \mathcal{P}} \!\!\left\{ \operatorname*{arg\,max}_{k \in \{i, j\}}  \mathbb{E}_{(\mathbf{x}, y) \sim \mathcal{D}} \!\!\left[ \left| \frac{\partial \mathcal{L}(f(\mathbf{x}), y)}{\partial x_k} \right| \right] \right\}, \\
    \mathcal{I}_{\Psi} &\!=\! \bigcup_{(i, j) \in \mathcal{P}} \!\!\left\{ \operatorname*{arg\,min}_{k \in \{i, j\}}  \mathbb{E}_{(\mathbf{x}, y) \sim \mathcal{D}} \!\!\left[ \left| \frac{\partial \mathcal{L}(f(\mathbf{x}), y)}{\partial x_k} \right| \right] \right\}.
\end{aligned}
\end{equation}
This competitive assignment mechanism ensures that for every correlated pair, the feature exhibiting robust predictive utility (maximal \textit{influence}) is assigned to $\mathcal{I}_{\Phi}$ for adversarial obfuscation, while its redundant counterpart (minimal \textit{influence}) is allocated to $\mathcal{I}_{\Psi}$ to serve as the carrier for the injected shortcut. We construct binary modulation masks $\mathbf{m}_{\Phi}, \mathbf{m}_{\Psi} \in \{0, 1\}^d$ via indicator functions on these sets. By construction, this guarantees structural orthogonality $\mathbf{m}_{\Phi} \odot \mathbf{m}_{\Psi} = \mathbf{0}$, effectively splitting the data manifold into conflicting control channels.

\textbf{Optimization Objective.}
To induce unlearnability, we formulate a \textit{Differential Gradient Steering} objective. The goal is to maximize the divergence in the obfuscation subspace while simultaneously converging the target spurious subspace towards a low-loss trajectory. Crucially, because these subspaces are naturally correlated, forcing them to diverge creates a highly correlated artifact, the \textit{Global Confounding Feature}, that artificially inflates the poison spectral norm $\lambda_p$. For a batch $\mathcal{B} = \{(\mathbf{x}_i, y_i)\}_{i=1}^N$, the data owner synthesizes the heterogeneous perturbation $\boldsymbol{\delta} = [\boldsymbol{\delta}_{\text{num}}, \boldsymbol{\delta}_{\text{cat}}]$ by maximizing the following surrogate objective:
\begin{equation}
\label{eq:UTOPIA_objective}
\begin{aligned}
\!\!\!\!\max_{\boldsymbol{\delta} \in \Omega(\mathbf{x})} \mathcal{J}_{\text{DSE}}(\boldsymbol{\delta}) \coloneqq &\frac{1}{N} \sum_{i=1}^N \Big[ 
 \underbrace{ \mathcal{L}\left( f(\mathbf{x}_i \oplus \boldsymbol{\delta}_i), y_i; \mathbf{m}_{\Phi} \right) }_{\text{Semantic Suppression (Ascent)}} \\
& - \lambda \cdot 
\underbrace{ \mathcal{L}\left( f(\mathbf{x}_i \oplus \boldsymbol{\delta}_i), y_i; \mathbf{m}_{\Psi} \right) }_{\text{Shortcut Injection (Descent)}} 
\Big],
\end{aligned}
\end{equation}
where $\mathcal{L}(\cdot; \mathbf{m})$ denotes the masked loss computed solely using gradients from the active mask $\mathbf{m}$, and $\lambda$ serves as the amplification factor. The operator $\oplus$ denotes type-specific perturbation: additive applied to numerical values and substitution applied to categorical tokens. 
During optimization, the surrogate model $f$ is fixed and initialized as the model pre-trained on the clean dataset.

\begin{algorithm}[t]\small
    \caption{Optimization of Our UTOPIA}
    \label{alg:UTOPIA}
    
    \LinesNotNumbered 
    \SetKwInOut{Input}{Input}
    \SetKwInOut{Output}{Output}
    \DontPrintSemicolon
    
    \Input{Dataset $\mathcal{D}_c = \{(\mathbf{x}_i, y_i)\}_{i=1}^N$, Surrogate model $f_{\theta}$, Correlation threshold $\tau$, Perturbation budget $\epsilon$, Iterations $T$, Amplification factor $\lambda$, Step size $\eta$.}
    \Output{Unlearnable dataset $\mathcal{D}_u = \{(\mathbf{x}_i \oplus \boldsymbol{\delta}_i, y_i)\}_{i=1}^N$}
    
    Train surrogate model $f_{\theta}$ on clean dataset $\mathcal{D}_c$ to converge\;
    
    \mycomment{Phase 1: Influence-guided subspace decoupling} \;
    
    Initialize matrix $R$ and pairs $\mathcal{P} \leftarrow \{(i,j) \mid |R_{ij}| > \tau\}$\;
    
    Calculate attribution scores $\mathcal{I}_k \leftarrow \mathbb{E}_{(x,y)\sim\mathcal{D}}[|\frac{\partial \mathcal{L}}{\partial x_k}|]$\;
    
    \For{pair $(i, j) \in \mathcal{P}$}{
        Assign dominant feature to $\mathcal{I}_{\Phi}: k_{\Phi} \!\leftarrow\! \underset{k \in \{i,j\}}{\arg\max}\; \mathcal{I}_k$\;
        Assign recessive feature to $\mathcal{I}_{\Psi}$: $k_{\Psi} \leftarrow \underset{k \in \{i,j\}}{\arg\min}\; \mathcal{I}_k$\;
    }
    
    Construct modulation masks $\mathbf{m}_{\Phi}$ and $\mathbf{m}_{\Psi}$ based on $\mathcal{I}_{\Phi}, \mathcal{I}_{\Psi}$\;
    
    \mycomment{Phase 2: Constraint-aware shortcut injection} \;
    
    \For{batch $(\mathbf{x}_i, y_i) \in \mathcal{D}_c$}{
        Initialize perturbations $\boldsymbol{\delta}_i \leftarrow \mathbf{0}$\;
        \For{iter $\leftarrow 1$ \KwTo $T$}{
            \mycomment{Compute gradient steering objective via Eq.~\ref{eq:UTOPIA_objective}} \;
            
            $\mathcal{L}_{\Phi} \!=\! \mathcal{L}\!\left(f(\mathbf{x}_i \!\oplus\! \boldsymbol{\delta}_i), y_i; \mathbf{m}_{\Phi}\right)$ \tcp*{Obfuscation}
            $~~\mathcal{L}_{\Psi} \!=\! \mathcal{L}\!\left(f(\mathbf{x}_i \oplus \boldsymbol{\delta}_i), y_i; \mathbf{m}_{\Psi}\right)$ \tcp*{Injection}
            
            $~~\mathcal{J}_{DSE} \leftarrow \mathcal{L}_{\Phi} - \lambda \cdot \mathcal{L}_{\Psi}$\;
            
            $~~\boldsymbol{\delta}_i \leftarrow \Pi_{\Omega(\mathbf{x}_i)}(\boldsymbol{\delta}_i + \eta \cdot \operatorname{sign}\!\left(\nabla_{\boldsymbol{\delta}_i}\, \mathcal{J}_{DSE}\right))$ \tcp*{Update}
        }
    }
    \Return $\mathcal{D}_u \leftarrow \{(\mathbf{x}_i \oplus \boldsymbol{\delta}_i, y_i)\}_{i=1}^N$
\end{algorithm}

We employ this \textbf{Counter-Directional Optimization} to fundamentally alter the spectral properties of the manifold: the first term performs gradient ascent on high-influence features to disrupt robust semantic dependencies, while the second term executes gradient descent on low-influence redundancies to inject a strong linear correlation. To strictly enforce domain-specific constraints within the heterogeneous optimization loop, we solve Eq.~\ref{eq:UTOPIA_objective} via a hybrid Projected Gradient Descent (PGD)~\cite{pgd}. At each iteration $t$, the perturbation is updated as:
\begin{equation}
\label{eq:pgd_update}
\!\!\boldsymbol{\delta}^{(t+1)} \!=\! \Pi_{\Omega(\mathbf{x})} \left[ \boldsymbol{\delta}^{(t)} \!+\! \eta \cdot \operatorname{sign} \left( \nabla_{\boldsymbol{\delta}} \mathcal{J}_{\text{DSE}}(\boldsymbol{\delta}^{(t)}) \right) \right],
\end{equation}
where $\eta$ is the step size and $\Pi_{\Omega(\mathbf{x})}$ is as a composite projection operator to respect the mixed-type manifold structure. Specifically, $\Pi_{\Omega(\mathbf{x})}$ constrains numerical perturbations to the intersection of the $\ell_{\infty}$-ball and intrinsic feature domains, while projecting categorical updates with greedy accumulated gradient strategy by following TabPGD~\cite{ben2024cafa}. Consequently, the synthesized artifact manifests as an optimal predictor satisfying $\lambda_p \gg \lambda_c$, compelling the unauthorized model to collapse onto this global confounder. The detailed optimization of UTOPIA is outlined in Alg.~\ref{alg:UTOPIA}.

\begin{figure}[t]
    \begin{minipage}[t]{0.48\linewidth}
        \centering
        \includegraphics[width=\textwidth]{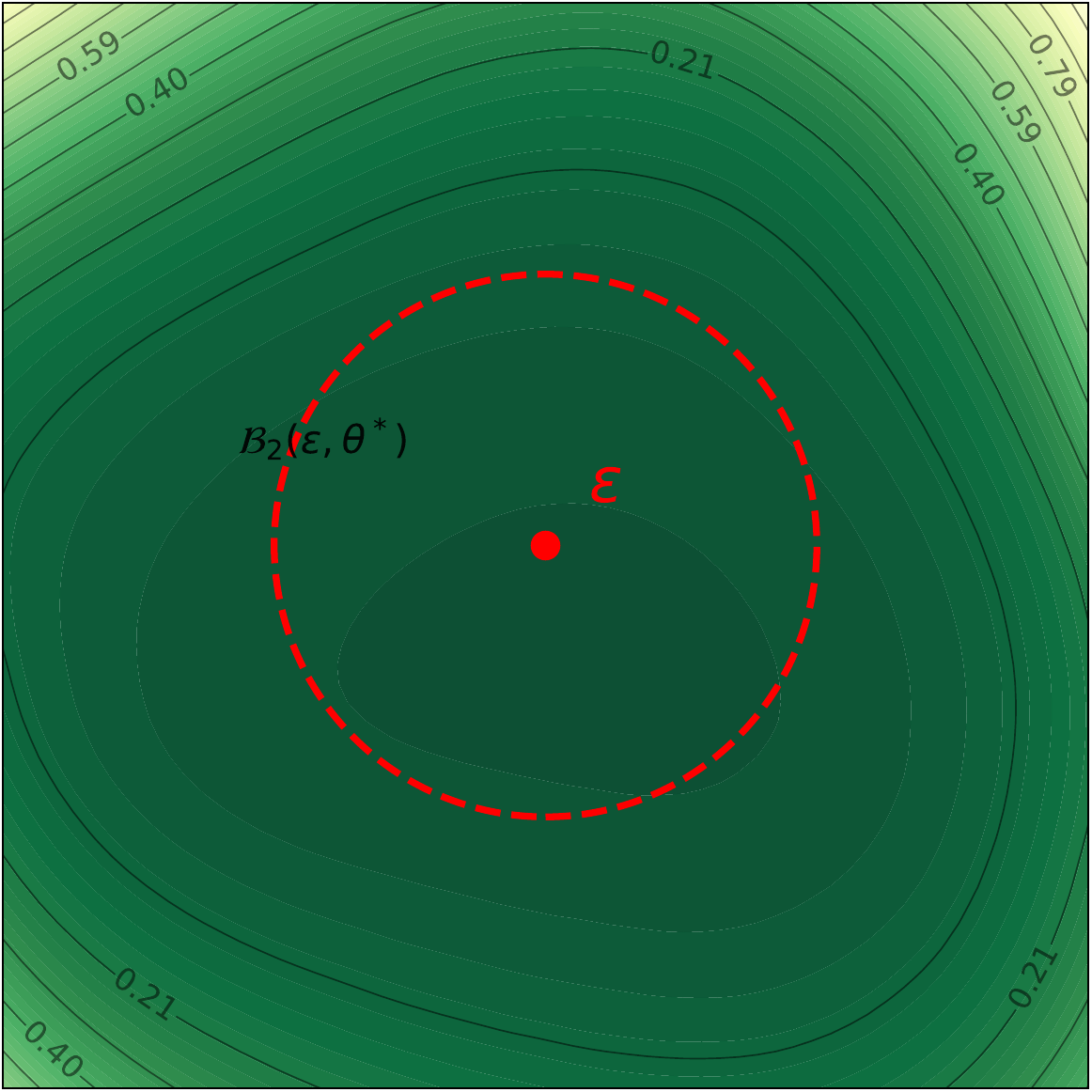}
        \centerline{\footnotesize(a) EM~\citep{em}}
    \end{minipage}%
    \hfill
    \begin{minipage}[t]{0.48\linewidth}
        \centering
        \includegraphics[width=\textwidth]{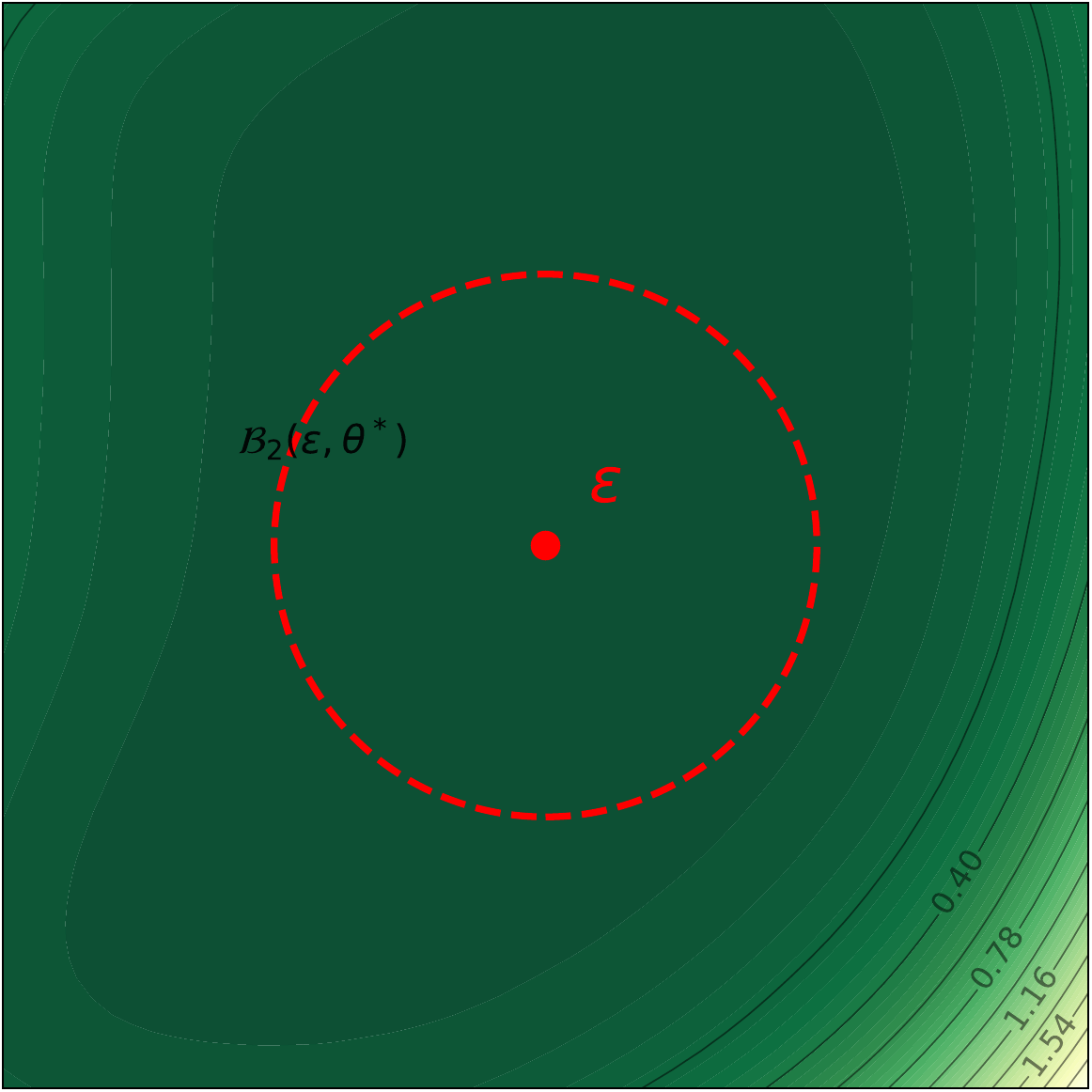}
        \centerline{\footnotesize(b) Our UTOPIA}
    \end{minipage}
    \vspace{-2mm}
    \caption{Comparison of convergence flatness, \textit{i.e.,} $\mathcal{L}_{\theta}$ Landscape.}
    \vspace{-3mm}
    \label{fig:flatness}
\end{figure}

\subsection{UTOPIA Can Induce Robust Basins}

\textbf{Analysis of Convergence Flatness.}
To empirically validate the impact of our UTOPIA on perturbation optimization dynamics, we visualize the loss landscape geometry around the converged model parameters, which enables us to qualitatively assess the local curvature and $\epsilon$-sharpness. We plot the loss surface variations along random direction vectors within an $\epsilon$-ball neighborhood $\mathcal{B}_2(\epsilon, \theta^*)$ centered at the optimal loss center, as shown in Fig.~\ref{fig:flatness}.
We find that the baseline EM~\citep{em} continuous optimization on discretized tabular data converges to a sharply peaked loss landscape, suggesting a minimum that is highly sensitive to parameter perturbations. By contrast, training with UTOPIA-perturbed data produces markedly flatter minima with more widely spaced loss contours. 
This geometric smoothing suggests that UTOPIA reshapes the optimization trajectory by replacing complex semantic dependencies with a simplified linear correlation (the global confounder), which forms a broad basin that steers convergence toward a robust spurious features rather than the true data distribution.

\begin{figure}[t]
    \centering
    \includegraphics[width=1\linewidth]{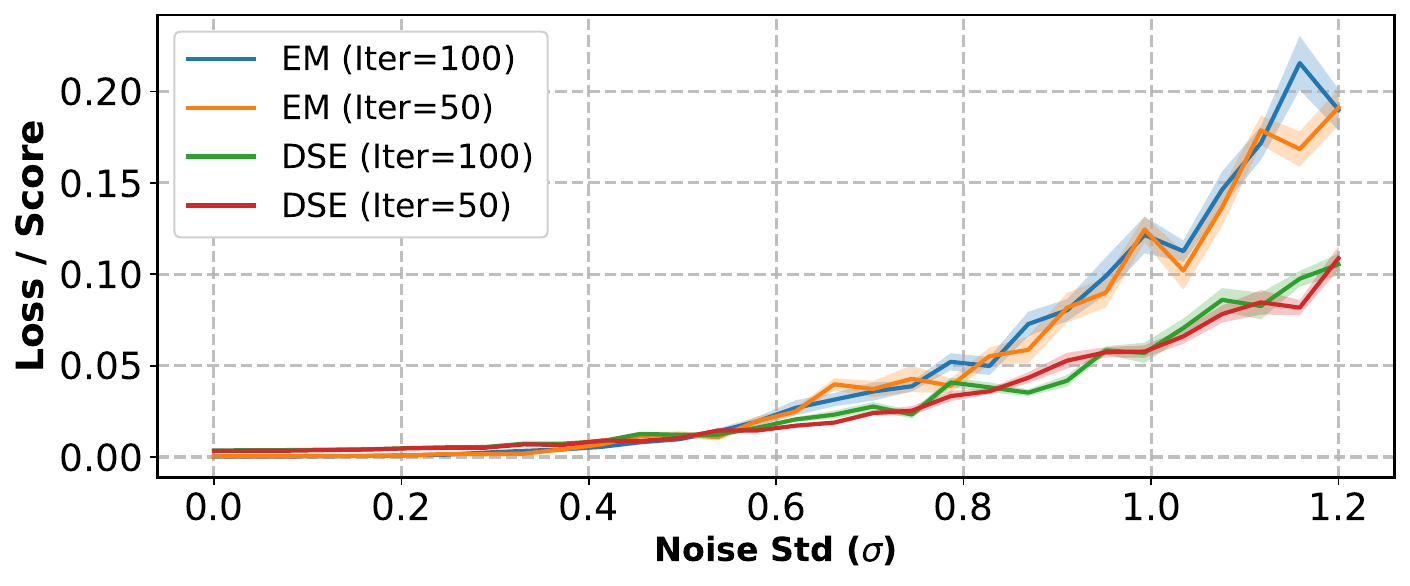}
    \vspace{-7mm}
    \caption{Comparison of noisy robustness.}
    \vspace{-4mm}
    \label{fig:noise_robustness}
\end{figure}

\textbf{Analysis of Noisy Robustness.}
Given the proven duality between parameter and input robustness~\citep{wu2020adversarial, zhang2024duality}, the observed convergence flatness implies high resistance to input perturbations. To validate this, we evaluate the model's stability against stochastic input noise ($\sigma$) in Fig.~\ref{fig:noise_robustness}. The baseline EM~\citep{em}  proves relatively unstable, exhibiting sharp loss spikes as noise increases. In contrast, UTOPIA maintains consistently lower training loss even under significant injection ($\sigma > 0.8$). This confirms that UTOPIA synthesizes robust, global confounders that dominate learning dynamics.

\section{Experiments}


\begin{table*}[t]\small
\centering
\captionsetup{font=small}
\caption{\textbf{Results on Binary Classification and Multi-classification Datasets:} We report average accuracy (\%) to compare with state-of-the-art UE baselines on Tabnet, FTT, Saint abd MLP. Our method is highlighted in the bottom row.}
\vspace{-2mm}
\scriptsize{
\resizebox{\linewidth}{!}{
\setlength\tabcolsep{3pt} 
\renewcommand\arraystretch{1.1} 
\begin{tabular}{l||cccc|cccc|cccc|cccc}
\hline\thickhline

\multicolumn{17}{l}{\textcolor{gray!90}{\textit{Binary Classification Tasks}}} \\ 

\rowcolor{gray20}
 & \multicolumn{4}{c|}{\textbf{CH}} & \multicolumn{4}{c|}{\textbf{KC1}} & \multicolumn{4}{c|}{\textbf{Churn}} & \multicolumn{4}{c}{\textbf{Employee}} \\
\cline{2-17} 

\rowcolor{gray20}
\multirow{-2}{*}{\textbf{Methods}} 
 & Tabnet & FTT & Saint & MLP & Tabnet & FTT & Saint & MLP & Tabnet & FTT & Saint & MLP & Tabnet & FTT & Saint & MLP \\
\hline\hline

Clean (Baseline) & 86.91 & 87.62 & 87.91 & 86.70 & 84.59 & 86.01 & 85.54 & 85.07 & 90.20 & 95.20 & 95.00 & 91.60 & 82.16 & 83.99 & 85.49 & 83.35 \\
\hdashline

\rowcolor{gray10}

CaFA~\cite{ben2024cafa} & 67.34 & 68.45 & 66.24 & 65.34 & 70.43 & 57.94 & 64.84 & 69.15 & 73.50 & 73.37 & 77.75 & 83.12 & 67.93 & 70.63 & 62.95 & 66.32 \\

TAP~\cite{tap}       & 64.95 & 70.76 & 69.43 & 73.76 & 72.15 & 63.42 & 68.01 & 73.40 & 72.50 & 75.25 & 83.25 & 85.10 & 69.49 & 67.52 & 65.56 & 68.52 \\

\rowcolor{gray10}
EM~\cite{em}         & 69.43 & 73.52 & 70.28 & 66.58 & 69.43 & 61.63 & 68.60 & 66.53 & 66.87 & 71.75 & 75.13 & 76.12 & 63.14 & 65.25 & 68.69 & 70.03 \\

RobustEM~\cite{rem}  & 63.34 & 71.26 & 63.40 & 70.73 & 60.85 & 54.61 & 62.04 & 59.65 & 63.50 & 66.50 & 58.87 & 80.87 & 59.87 & 59.72 & 64.66 & 60.03 \\

\rowcolor{gray10}
SEP~\cite{sep}        & 62.70 & 61.54 & 64.65 & 67.56 & 59.24 & 56.26 & 58.70 & 59.65 & 62.46 & 65.12 & 65.87 & 73.37 & 57.16 & 59.84 & 60.62 & 63.41  \\

PUE~\cite{wang2025provably} & 59.92 & 58.46 & 61.79 & 63.00 & 58.65 & 55.55 & 59.10 & 58.76 & 61.50 & 64.12 & 60.87 & 75.62 & 58.68 & 64.63 & 58.54 & 59.59 \\

\hline

\rowcolor{highlightcyan}
\textbf{UTOPIA (Ours)} & \textbf{47.51} & \textbf{45.36} & \textbf{49.75} & \textbf{50.17} & \textbf{45.79} & \textbf{40.37} & \textbf{46.54} & \textbf{41.89} & \textbf{42.15} & \textbf{52.11} & \textbf{45.25} & \textbf{61.02} & \textbf{47.44} & \textbf{44.36} & \textbf{45.86} & \textbf{49.26} \\
\hline\hline

\multicolumn{17}{l}{\textcolor{gray!90}{\textit{Multi-class Classification Tasks}}} \\ 

\rowcolor{gray20}
 & \multicolumn{4}{c|}{\textbf{JV}} & \multicolumn{4}{c|}{\textbf{IF}} & \multicolumn{4}{c|}{\textbf{Dry Bean}} & \multicolumn{4}{c}{\textbf{EOL}} \\
\cline{2-17}

\rowcolor{gray20}
\multirow{-2}{*}{\textbf{Methods}} 
 & Tabnet & FTT & Saint & MLP & Tabnet & FTT & Saint & MLP & Tabnet & FTT & Saint & MLP & Tabnet & FTT & Saint & MLP \\
\hline\hline

Clean (Baseline) & 94.58 & 96.88 & 98.49 & 97.19 & 72.43 & 79.74 & 80.27 & 79.72 & 91.77 & 92.36 & 92.17 & 92.10 & 88.66 & 95.98 & 94.56 & 86.28 \\
\hdashline

\rowcolor{gray10}

CaFA~\cite{ben2024cafa} & 49.30 & 51.72 & 57.19 & 53.51 & 62.76 & 67.43 & 72.61 & 68.13 & 43.98 & 42.03 & 45.98 & 51.54 & 44.55 & 39.71 & 39.85 & 43.38 \\

TAP~\cite{tap}       & 41.33 & 55.36 & 51.61 & 44.69 & 51.38 & 71.37 & 73.52 & 72.43 & 46.16 & 41.54 & 45.15 & 55.15 & 52.63 & 49.08 & 44.52 & 51.45 \\

\rowcolor{gray10}
EM~\cite{em}         & 43.07 & 46.78 & 52.62 & 38.02 & 51.64 & 65.41 & 75.19 & 73.67 & 35.33 & 34.64 & 40.61 & 47.36 & 40.79 & 34.88 & 35.47 & 35.71 \\

RobustEM~\cite{rem}  & 31.85 & 36.05 & 47.27 & 34.19 & 50.38 & 59.09 & 64.59 & 65.84 & 33.35 & 31.74 & 35.03 & 43.65 & 35.17 & 26.89 & 31.33 & 29.89 \\

\rowcolor{gray10}
SEP~\cite{sep}       & 34.66 & 40.37 & 50.18 & 37.24 & 51.09 & 50.87 & 58.12 & 63.72 & 31.64 & 28.91 & 35.73 & 42.26 & 28.01 & 25.96 & 30.68 & 31.38 \\

PUE~\cite{wang2025provably}  & 31.60 & 36.06 & 48.85 & 34.11 & 44.21 & 54.43 & 57.31 & 58.87 & 34.12 & 30.32 & 38.61 & 35.35 & 30.45 & 27.26 & 31.57 & 31.27 \\

\hline

\rowcolor{highlightcyan}
\textbf{UTOPIA (Ours)} & \textbf{15.24} & \textbf{21.32} & \textbf{31.05} & \textbf{21.39} & \textbf{31.47} & \textbf{34.28} & \textbf{41.48} & \textbf{42.39} & \textbf{17.31} & \textbf{13.65} & \textbf{18.98} & \textbf{25.36} & \textbf{15.38} & \textbf{14.20} & \textbf{17.75} & \textbf{18.34} \\
\hline

\end{tabular}}}

\label{tab:styled_final}
\vspace{-5mm}
\end{table*}


\textbf{Models and Datasets.}
In our experiments, we employ 4 representative architectures: FT-Transformer (FTT) \cite{dai2025ft}, SAINT \cite{somepalli2021saint}, MLP, and TabNet \cite{arik2021tabnet}. To evaluate the transferability, we use 4 models as the surrogate model to generate UE, while other 10 models including AutoInt \cite{auroint}, ExcelFormer \cite{ExcelFormer}, RealMLP \cite{Realmlp}, ResNet \cite{resnet}, SNN \cite{snn}, TabM \cite{Tabm}, TabR \cite{TabR}, TabTransformer \cite{tabtransformer}, Trompt \cite{Trompt}, and XGBoost \cite{chen2016xgboost} are utilized as the target models for cross-architecture evaluation. And we conduct experiments on eight tabular datasets: California Housing Classification (CH) \cite{California_Housing_Classification}, Churn \cite{churn_openml_dataset} , KC1 \cite{kc1}, Employee \cite{employee_dataset}, Dry Bean \cite{dry_bean_602}, Internet Firewall (IF) \cite{internet_firewall_data_542}, Japanese Vowels (JV) \cite{japanese_vowels_128} and Estimation of Obesity Levels (EOL) \cite{estimation_of_obesity_levels}, with more details in Appendix~\ref{appendix:dataset}.

\textbf{Implementation Details.}
Following \citet{em,yu2025mtl}, we adopt a full protection strategy where the rate of protection data is set to 100\%. Following \citet{ben2024cafa}, the perturbation budget $\epsilon$ is restricted to $0.03$ under the $\ell_\infty$ norm to ensure the imperceptibility. We employ PGD~\cite{pgd} with iterations $N=20$ to optimize the perturbations. We set the amplification factor to $\lambda = 5.0$ and correlation threshold is $\tau = 0.5$. Additionally, the momentum coefficient $\mu$ is set to $1.0$. For the training of all models, we set epochs to $30$, and use AdamW optimizer with $1.0 \times 10^{-3}$ learning rate and cosine annealing LR scheduler ($T=30$).

\textbf{Evaluation Metrics.}
For all classification tasks in our experiments, we report the accuracy (\%) on the clean test set. In the context of UE, such accuracy is an indicator of the protection strength: a lower accuracy signifies a superior unlearnable effect, showing that the model has failed to extract generalizable features from the protected training data.

\textbf{Baselines.}
We select 5 UE baseline including: TAP~\cite{tap}, EM~\cite{em} , RoustEM~\cite{rem}, SEP~\cite{sep} and PUE~\cite{wang2025provably}.
In addition, we report results for the test-time adversarial attack CaFA~\cite{ben2024cafa} as a reference.

\subsection{Experimental Results}


\textbf{Results on Binary Classification Tasks.} 
Table~\ref{tab:styled_final} shows that our UTOPIA consistently achieves the most performance degradation across all binary classification datasets, often pushing the model accuracy toward or even below the random guessing threshold (50\%). Specifically, on the \textit{KC1} dataset, UTOPIA reduces the accuracy of the MLP model to 41.89\%, whereas the strongest baseline, PUE, only manages to lower it to 58.76\%. A key finding from the binary tasks is the architectural robustness of UTOPIA. While existing methods like SEP and PUE show fluctuating effectiveness, UTOPIA maintains a stable unlearnable effect across all four architectures, typically outperforming the best baseline. This suggests that UTOPIA successfully captures core tabular features that are indispensable for binary decision boundaries, effectively protecting data from unauthorized learning process.

\begin{table}[t]
\centering
\captionsetup{font=small}
\caption{\textbf{Impact of Perturbation Budget ($\epsilon$)}: Accuracy (\%) across datasets and models.}
\vspace{-2mm}
\scriptsize{
\resizebox{\columnwidth}{!}{
\setlength\tabcolsep{1pt} 
\renewcommand\arraystretch{1.2} 
\begin{tabular}{l||cc|cc|cc|cc}
\hline\thickhline
\rowcolor{gray20}
\textbf{Budget} $\rightarrow$ & \multicolumn{2}{c|}{\bm{$\epsilon=0.007$}} & \multicolumn{2}{c|}{\bm{$\epsilon=0.009$}} & \multicolumn{2}{c|}{\bm{$\epsilon=0.01$}} & \multicolumn{2}{c}{\bm{$\epsilon=0.03$}} \\
\cline{2-9} 
\rowcolor{gray20}
\textbf{Model} $\downarrow$, \textbf{Dataset} $\rightarrow$ & JV& IF& JV & IF & JV & IF & JV & IF \\ 
\hline\hline

Tabnet & 38.71 & 59.69 & \textbf{30.95} & \textbf{53.26} & \textbf{27.62} & \textbf{45.90} & \textbf{15.24} & \textbf{31.47} \\

\rowcolor{gray10}
FTT    & \textbf{37.10} & \textbf{57.37} & 35.65 & 55.76 & 28.21 & 48.78 & 21.32 & 34.28  \\

Saint  & 55.50 & 67.12 & 50.90 & 59.60 & 46.76 & 55.93 & 31.05 & 41.48  \\

\rowcolor{gray10}
MLP    & 42.27 & 60.47 & 40.32 & 57.64 & 37.64 & 57.43 & 21.39 & 42.39  \\

\hline
\end{tabular}
}}
\label{tab:perturbation_budget}
\vspace{-4mm}
\end{table}

\textbf{Results on Multi-class Classifications Tasks.} 
In the multi-class settings as shown in Table~\ref{tab:styled_final},UTOPIA frequently reaching near-random performance levels. On the \textit{Dry Bean} dataset, UTOPIA achieves a remarkable low of 13.65\% accuracy on the FTT model, compared to 28.91\% achieved by SEP. We observe a recurring phenomenon where transformer-based models (Saint and FTT), which are typically robust to noise, still succumb to UTOPIA's perturbations. Notably, on \textit{JV} dataset, the Clean baseline reaches 98.49\% accuracy on Saint, and UTOPIA successfully suppresses this to 31.05\%, a significantly lower result than RobustEM's 47.27\%. Our findings indicate that as the number of classes increases, UTOPIA exploits the increased complexity of the label space more effectively than baselines. The consistent results across \textit{EOL} and \textit{JV} dataset show that UTOPIA provides a universal unlearnable solution that is not sensitive to the specific number of categories or the inherent dimensionality of the tabular data. We also conduct evaluations on computational cost in Appendix~\ref{appendix:generation_time}.

\begin{table}[t]
\centering
\captionsetup{font=small}
\caption{\textbf{Transferability}: Accuracy (\%) across different model architectures and datasets.}
\vspace{-2mm}
\scriptsize{
\resizebox{1.0\columnwidth}{!}{
\setlength\tabcolsep{1.0pt}
\renewcommand\arraystretch{1.05}
\begin{tabular}{l|l|| cccccccccc}
\hline\thickhline
\rowcolor{gray20}
\textbf{Dataset} & \textbf{From $\downarrow$} & \multicolumn{10}{c}{\textbf{Target Model}} \\ 
\cline{3-12}
\rowcolor{gray20}
$\downarrow$ &  \textbf{To $\rightarrow$}& \textbf{AI} & \textbf{EF} & \textbf{RM} & \textbf{RN} & \textbf{SN} & \textbf{TM} & \textbf{TR} & \textbf{TT} & \textbf{TP} & \textbf{XG} \\ 
\hline\hline

& FTT    & \textbf{54.92} & \textbf{53.83} & \textbf{55.98}          & \textbf{52.17} & 52.52          & 54.78          & 56.19          & 59.34          & 62.10 & 61.88 \\
\rowcolor{gray10}
\cellcolor{white} & MLP    & 59.77          & 60.65          & 56.95 & 57.19          & 58.27          & 59.09          & \textbf{55.80} & 63.76          & 60.16 & \textbf{57.28} \\
& SAINT  & 56.22          & 60.22          & 59.15          & 62.80          & 56.96          & 58.71          & 61.99          & \textbf{53.98} & 66.08          & 60.25 \\
\rowcolor{gray10}
\cellcolor{white} \multirow{-4}{*}{\shortstack[l]{CH}} & TabNet & 57.43          & 56.17          & 53.27          & 58.68          & \textbf{48.74} & \textbf{43.94} & 57.12          & 57.55 & \textbf{56.02} & 59.88 \\

\hline

& FTT    & 24.86          & 26.49 & 29.04          & 27.79         & 31.89          & 27.71          & 26.68          & 32.71          & 28.71          & \textbf{26.95} \\
\rowcolor{gray10}
\cellcolor{white} & MLP    & 23.38          & 23.60          & 24.66          & 27.68          & 26.70          & 28.33          & \textbf{23.46} & \textbf{22.67} & 30.68          & 34.55 \\
& SAINT  & \textbf{21.60}          & \textbf{22.74} & \textbf{20.42} & \textbf{24.14}          & 33.89          & \textbf{20.06} & 29.17          & 26.66          & \textbf{25.98}          & 27.31 \\
\rowcolor{gray10}
\cellcolor{white} \multirow{-4}{*}{\shortstack[l]{EOL}} & TabNet & 28.27          & 25.31          & 24.55          & 26.63 & \textbf{23.38}          & 24.18          & 27.13          & 29.08          & 32.75          & 29.97 \\

\hline

& FTT    & 31.88 & 33.36          & \textbf{32.30} & \textbf{31.68} & 35.57          & 37.15          & 29.57          & 37.69          & 31.98          & 37.51 \\
\rowcolor{gray10}
\cellcolor{white} & MLP    & 34.93          & 31.01          & 36.15          & 36.98          & 32.95          & 34.46          & 38.94          & 41.43         & 32.33          & 35.04 \\
& SAINT  & 42.03          & 48.06          & 41.75          & 40.96          & 39.72          & 44.13          & 45.97          & 36.46          & 46.34          & 37.76 \\
\rowcolor{gray10}
\cellcolor{white} \multirow{-4}{*}{\shortstack[l]{JV}} & TabNet & \textbf{28.08}          & \textbf{22.41} & 34.95          & 28.83          & \textbf{25.42} & \textbf{28.03} & \textbf{24.46} & \textbf{21.36} & \textbf{24.45} & \textbf{29.84} \\

\hline
\end{tabular}
}}
\vspace{-2mm}
\begin{flushleft}
\scriptsize \textbf{Target Models}: \textbf{AI}: AutoInt, \textbf{EF}: ExcelFormer, \textbf{RM}: RealMLP, \textbf{RN}: ResNet, \textbf{SN}: SNN, \textbf{TM}: TabM, \textbf{TR}: TabR, \textbf{TT}: TabTransformer, \textbf{TP}: Trompt, \textbf{XG}: XGBoost.
\end{flushleft}
\label{tab:transferability}
\vspace{-3mm}
\end{table}

\begin{table}[t]
\centering
\captionsetup{font=small}
\caption{\textbf{Ablation study}: Accuracy(\%) across datasets and models:}
\label{tab:ablation}
\vspace{-2mm}
\scriptsize{
\resizebox{1.0\columnwidth}{!}{
\setlength\tabcolsep{2pt} 
\renewcommand\arraystretch{1.05} 
\begin{tabular}{l||cccc|cccc}
\hline\thickhline
\rowcolor{gray20}
\textbf{Dataset} $\rightarrow$ & \multicolumn{4}{c|}{\textbf{JV}} & \multicolumn{4}{c}{\textbf{IF}} \\
\cline{2-9} 
\rowcolor{gray20}
\textbf{Model} $\rightarrow$ & Tabnet & FTT & Saint & MLP & Tabnet & FTT & Saint & MLP \\ 
\hline\hline

\ding{172} w/o Decoupling & 24.03 & 29.27 & 41.21 & 28.48 & 45.77 & 55.21 & 51.22 & 47.61 \\

\rowcolor{gray10}

\ding{173} w/o Suppression & 30.30 & 32.08 & 50.94 & 35.38 & 56.86 & 59.25 & 56.63 & 63.61 \\

\ding{174} w/o Injection & 32.67 & 35.89 & 45.33 & 31.17 & 50.88 & 63.14 & 48.86 & 55.23 \\

\rowcolor{gray10}
\ding{175} w/o Influence & 26.35 & 27.33 & 37.32 & 25.21 & 44.19 & 50.91 & 61.11 & 46.58 \\

\hline
\rowcolor{highlightcyan}
\textbf{Our UTOPIA} & \textbf{15.24} & \textbf{21.32} & \textbf{31.05} & \textbf{21.39} & \textbf{31.47} & \textbf{34.28} & \textbf{41.48} & \textbf{42.39} \\ 
\hline
\end{tabular}
}}
\vspace{-2mm}
\end{table}

\begin{figure*}[t] 
    \centering
    \includegraphics[width=\textwidth]{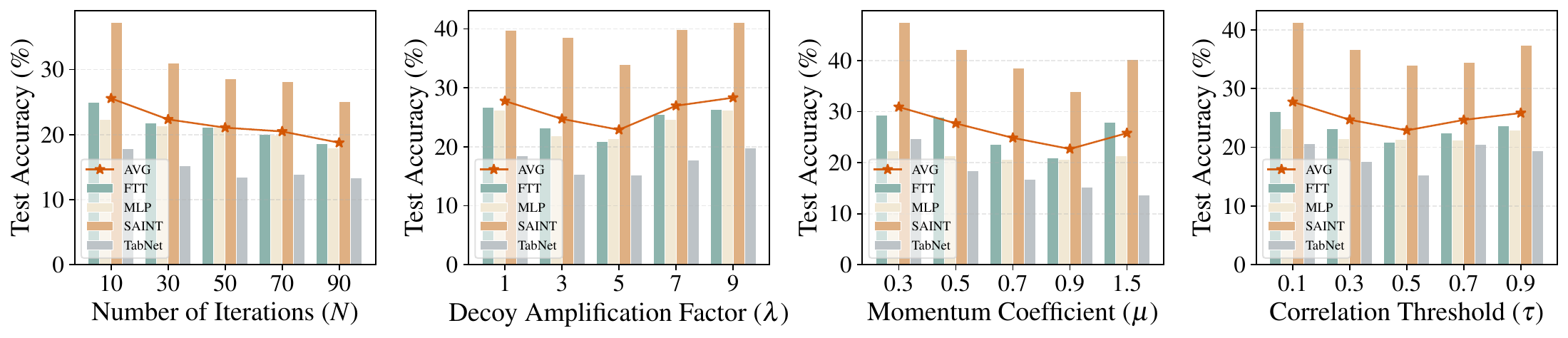}
    \vspace{-7mm}\caption{\textbf{Hyperparameter sensitivity analysis on JV dataset:} The orange line denotes the average performance across all models.}
    \label{fig:hyperparameters}
    \vspace{-1mm}
\end{figure*}

\begin{table*}[t]\small
\centering
\captionsetup{font=small}
\caption{\textbf{Comparison of different Defense Results (Accuracy \%):} The table evaluates various defense mechanisms across three datasets using four different model architectures. The bold value denotes the highest under defenses.}
\vspace{-2mm}
\scriptsize{
\resizebox{\linewidth}{!}{
\setlength\tabcolsep{5pt} 
\renewcommand\arraystretch{1.1} 
\begin{tabular}{l||cccc|cccc|cccc}
\hline\thickhline
\rowcolor{gray20}
 & \multicolumn{4}{c|}{\textbf{KC1}} & \multicolumn{4}{c|}{\textbf{JV}} & \multicolumn{4}{c}{\textbf{Churn}} \\
\cline{2-13} 

\rowcolor{gray20}
\multirow{-2}{*}{\textbf{Defenses} $\downarrow$} 
 & Tabnet & FTT & Saint & MLP & Tabnet & FTT & Saint & MLP & Tabnet & FTT & Saint & MLP \\
\hline\hline

\textbf{No Defense} & 45.79 & 40.37 & 46.54 & 41.89 & 15.24 & 21.32 & 31.05 & 21.39 & 42.15 & 52.11 & 45.25 & 61.02 \\
\hdashline

\rowcolor{gray10}
Feature Masking    & 58.28 & 54.43 & 53.55 & 52.36 & 17.87 & 21.39 & 39.08 & 22.02 & 50.75 & 56.50 & 52.12 & 68.37 \\

Feature Squeezing  & 49.40 & 56.21 & 57.10 & 50.00 & 20.32 & 23.77 & 46.80 & 23.14 & 55.87 & 59.37 & 58.67 & 70.25 \\

\rowcolor{gray10}
Gaussian Noise     & 53.25 & 52.66 & 47.04 & 46.15 & 22.45 & 25.03 & 47.80 & 25.65 & 63.25 & 63.75 & 50.37 & 69.50 \\

Label Smoothing    & 56.80 & 49.70 & 58.87 & 53.84 & 18.69 & \textbf{25.53} & 49.87 & 23.58 & 47.75 & 58.00 & 55.12 & \textbf{70.53} \\

\rowcolor{gray10}
Mixup              & \textbf{59.46} & \textbf{62.42} & \textbf{57.98} & \textbf{58.57} & \textbf{26.41} & 25.40 & \textbf{54.20} & \textbf{32.62} & 61.50 & \textbf{65.62} & \textbf{61.12} & 72.37 \\

Quantization       & 54.43 & 42.01 & 51.47 & 44.67 & 23.83 & 24.84 & 46.80 & 25.47 & \textbf{63.50} & 53.37 & 57.62 & 66.75 \\

\rowcolor{gray10}
Swap Noise         & 48.81 & 47.33 & 48.52 & 42.60 & 20.95 & 23.32 & 42.28 & 24.02 & 62.12 & 57.62 & 56.75 & 65.51 \\
\hline

\end{tabular}}}
\label{tab:defense_results}
\vspace{-3mm}
\end{table*}

\textbf{Impact of Perturbation Budget.}
We evaluate the impact of varying perturbation budgets by testing the models under $\epsilon = 0.007, 0.009, 0.01,$ and $0.03$. As shown in Table~\ref{tab:perturbation_budget}, the performance across all models generally decreases as the perturbation budget increases. Tabnet consistently achieves the best performance across most budget levels, while architectures like Saint show a more pronounced sensitivity to larger perturbation magnitudes.

\textbf{Transferability.}
We assess the transferability of our UE in a black-box setting, where the attacker employs a model architecture different from the surrogate model used to generate UE. We utilize the four models from our main experiment as surrogates to generate UE. The results are reported in Table~\ref{tab:transferability}. Our method demonstrates strong transferability across all datasets, successfully degrading the accuracy of the 10 distinct target models.

\textbf{Ablation Studies.}
The results are presented in Table \ref{tab:ablation}. Compared to the full UTOPIA, \ding{172} demonstrates that disabling subspace partitioning leads to a consistent increase in accuracy, confirming the necessity of decoupling features. The most significant performance drop occurs in \ding{173}, which highlights that semantic suppression via gradient ascent is a critical component for successfully obfuscating the model's decision making process. While \ding{174} shows that shortcut injection via gradient descent further enhances attack effectiveness, the results in \ding{175} emphasize the importance of our influence guidance strategy; replacing saliency based assignment with random partitioning results in a clear loss of performance across all models. Overall, the full UTOPIA method achieves the lowest accuracy across both datasets, proving that the synergy between decoupling, suppression, and guided injection is essential for a robust attack.

\textbf{Hyper-Parameter Analysis.}
We analyze four hyperparameters $N$, $\alpha$, $\mu$, and $\tau$ as shown in Fig.~\ref{fig:hyperparameters}. 
The influence of the number of iterations $N$ demonstrates a consistent trend where increasing iterations improves the attack performance. For the amplification factor $\alpha$, the optimal performance is achieved at 5. We attribute this to the need for a balanced amplification that disturbs the learning process without dominating the optimization, deviations from this value result in diminished effectiveness. Similar to the observations for the other parameters, moderate variations in $\mu$ and $\tau$ do not drastically alter the effect.

\textbf{Robustness.}
To evaluate the resistance of our proposed UE, we employ various data augmentation and purification techniques as defense mechanisms, attempting to neutralize the injected perturbations and restore model training. We test seven distinct methods, including Feature Masking, Feature Squeezing, Gaussian Noise, Label Smoothing, and Mixup, Quantization and Swap Noise across multiple datasets and architectures. Results in Table~\ref{tab:defense_results} reveal that while applying these augmentation techniques can yield a marginal improvement in accuracy, the overall performance uplift remains limited and fails to recover the semantic utility of the data. Among the evaluated defenses, Mixup emerges as the relatively most effective strategy, consistently achieving higher accuracy compared to other methods.

\vspace{-1mm}
\section{Conclusion}
\vspace{-1mm}
This paper presents UTOPIA, a framework tailored for generating unlearnable examples on tabular data. Instead of directly applying global-level perturbations, we employ influence-guided subspace decoupling to partition features into conflicting optimization trajectories. Through differential gradient steering, UTOPIA suppresses semantic learnability in high-saliency subspaces while simultaneously injecting dominant shortcuts into redundant ones. Furthermore, by applying constraint-aware optimization, we ensure that perturbations strictly adhere to the structural validity of mixed-type domains. Extensive experiments demonstrate UTOPIA's effectiveness and robust transferability across diverse architectures in safeguarding tabular data privacy.








\section*{Impact Statement}
This paper presents a method for protecting sensitive tabular datasets (e.g., finance, healthcare, and user-behavior logs) from unauthorized model training by releasing constraint-preserving unlearnable data. The intended positive impact is stronger data stewardship: it can reduce incentives for large-scale scraping, offer a practical layer of protection when access control is imperfect, and help organizations share data for approved analysis while limiting illicit reuse. At the same time, similar techniques could be misused to sabotage benign training pipelines, poison community benchmarks, or create disputes about data integrity, potentially increasing costs and eroding trust in shared resources. To mitigate these risks, the method should be positioned as a defensive tool for legitimate data owners, paired with clear dataset labeling, controlled distribution, and auditing. Continued work on detection, provenance tracking, and robust training protocols can further limit misuse while preserving the benefits.


\bibliography{example_paper}
\bibliographystyle{icml2026}

\newpage
\appendix
\onecolumn
\section{Proofs.}

In this section, we provide the proofs of our theoretical results in Section \ref{sec:theory}.

\subsection{Proof for Theorem  \ref{thm:analytical_suppression}}
\label{prf:analytical_suppression}

\begin{proof}

The fixed point $\boldsymbol{\alpha}^*$ of the dual dynamics satisfies a nonlinear system involving the feature correlation matrix $\mathbf{K} = \mathbf{X}\mathbf{X}^\top$. In the large-sample limit, $\mathbf{K}$ concentrates around the population covariance
\[
   \mathbf{\Sigma}_x \coloneqq \mathbb{E}[\mathbf{x}\mathbf{x}^\top]
   = \mathbf{\Sigma}_c \oplus \mathbf{\Sigma}_p,
\]
which is block-diagonal with respect to the $(\mathcal{S}_c,\mathcal{S}_p)$ decomposition. 
Following the analysis of \citet{pezeshki2021gradient}, the equilibrium primal weights along an eigendirection with eigenvalue $\lambda$ can be characterized as:
\begin{equation}
   w^*(\lambda) 
   \;=\; 
   \mathcal{A}(\lambda,\gamma)\,\xi,
\end{equation}
where $\xi$ encodes the label alignment (corresponding to the dual variable magnitude) and $\mathcal{A}(\lambda,\gamma)$ is a spectral response function expressible via the Lambert $W$ function.

Crucially, the clean and poison components compete to fit the labels. The effective gradient signal available to the clean subspace is suppressed by the presence of the dominant poison subspace. Specifically, the response function for the clean component, $\mathcal{A}(\lambda_c, \gamma)$, is attenuated by a factor involving the total signal strength $\lambda_c + \lambda_p$. Analytical comparison yields the scaling:
\begin{equation}
   \|\mathbf{w}_c^*\|_2 
   \;\le\; 
   \frac{\xi}{1 + \lambda_p / \lambda_c + \gamma / \lambda_c}\, g(\lambda_c)
   \;\le\;
   \frac{\xi}{\kappa}\, g(\lambda_c),
\end{equation}
for a strictly increasing function $g(\lambda_c)$ that depends only on the clean spectrum.
Parameterizing $g(\lambda_c) = W(\psi(\lambda_c))$ for convenience, we obtain the bound in \eqref{eq:lambert_bound}. The factor $1/\kappa$ explicitly captures the spectral dominance: as the poison features become stronger ($\kappa \to \infty$), the model relies on them almost exclusively, driving the clean weights to zero.
\end{proof}

\subsection{Proof for Theorem  \ref{thm:certified_bound}}
\label{prf:certified_bound}

\begin{proof}
For a clean sample $\mathbf{x} \in \mathcal{S}_c$ (so that $\mathbf{x}_p = 0$), we have $\mathbf{w}_p^{*\top}\mathbf{x} = 0$ and hence
\begin{equation}
   M(\mathbf{x}) 
   \;=\; 
   y\, (\mathbf{w}^* + \boldsymbol{\delta})^\top \mathbf{x}
   \;=\;
   y\, (\mathbf{w}_c^* + \boldsymbol{\delta})^\top \mathbf{x},
   \ \boldsymbol{\delta} \sim \mathcal{N}(\mathbf{0}, \sigma^2 \mathbf{I}).
\end{equation}
Conditioned on $\mathbf{x}$, the random variable $M(\mathbf{x})$ is Gaussian with mean $y\,\mathbf{w}_c^{*\top}\mathbf{x}$ and variance $\sigma^2 \|\mathbf{x}\|_2^2$. Hence the probability of correct classification can be written exactly as
\begin{equation}
   \mathbb{P}(\tilde{f}(\mathbf{x}) = y)
   \;=\;
   \Phi\!\left(
       \frac{y\,\mathbf{w}_c^{*\top}\mathbf{x}}{\sigma \|\mathbf{x}\|_2}
   \right).
\end{equation}
By the Cauchy--Schwarz inequality,
\begin{equation}
   \left|
       \frac{y\,\mathbf{w}_c^{*\top}\mathbf{x}}{\sigma \|\mathbf{x}\|_2}
   \right|
   \;\le\;
    \frac{\|\mathbf{w}_c^*\|_2}{\sigma}.
\end{equation}
Since $\Phi(\cdot)$ is monotonically increasing, we obtain the uniform upper bound
\begin{equation}
   \mathbb{P}(\tilde{f}(\mathbf{x}) = y)
  \;\le\;
   \Phi\!\left( \frac{\|\mathbf{w}_c^*\|_2}{\sigma} \right)
\end{equation}
for all clean $\mathbf{x}$. Taking expectation over $(\mathbf{x},y)\sim\mathcal{D}_{\text{clean}}$ yields
\begin{equation}
   \underset{(\mathbf{x}, y) \sim \mathcal{D}_{\text{clean}}}{\mathbb{E}} 
   \left[ \mathbb{P}(\tilde{f}(\mathbf{x}) = y) \right]
   \;\le\;
   \Phi\!\left( \frac{\|\mathbf{w}_c^*\|_2}{\sigma} \right).
\end{equation}
Applying Theorem~\ref{thm:analytical_suppression}, there exists a constant $\mathcal{C}>0$ such that
\begin{equation}
   \|\mathbf{w}_c^*\|_2 
   \;\le\; 
   \frac{\mathcal{C}}{\kappa}\, W\!\bigl(\psi(\lambda_c)\bigr),
\end{equation}
which gives the desired bound
\begin{equation}
   \underset{(\mathbf{x}, y) \sim \mathcal{D}_{\text{clean}}}{\mathbb{E}} 
   \left[ \mathbb{P}(\tilde{f}(\mathbf{x}) = y) \right] 
   \;\le\; 
   \Phi\left( 
       \frac{\mathcal{C}}{\sigma} \cdot 
       \frac{W(\psi(\lambda_c))}{\kappa}
   \right).
\end{equation}
For fixed $\lambda_c$ and $\sigma$, the quantity $W(\psi(\lambda_c))$ is a constant independent of $\kappa$, so the argument of $\Phi(\cdot)$ tends to $0$ as $\kappa \to \infty$, and the accuracy bound converges to
\begin{equation}
   \lim_{\kappa \to \infty}
   \Phi\left( 
       \frac{\mathcal{C}}{\sigma} \cdot 
       \frac{W(\psi(\lambda_c))}{\kappa}
   \right)
   \;=\;
   \Phi(0)
   \;=\;
   0.5.
\end{equation}
This certifies that, under sufficiently strong spectral dominance, the smoothed classifier cannot provably achieve accuracy better than random guessing on the clean distribution.
\end{proof}

\section{Dataset.}
\label{appendix:dataset}

In Table \ref{tab:generation_time}, we report the computational time (in seconds) required to generate unlearnable examples for each dataset and model configuration. These results provide an overview of the efficiency of our method.

\begin{table}[th]
\centering
\caption{Summary of Dataset Statistics and Characteristics: The table provides an overview of the eight datasets used in our experiments, including the number of classes, numerical ($N_{num}$) and categorical ($N_{cat}$) features, and the sizes of training, validation, and test sets.}
\vspace{-2mm}
\label{tab:dataset_stats}
\setlength{\tabcolsep}{2.0pt} 
\begin{tabular}{l c c c c c c}
\toprule
\textbf{Dataset} & \textbf{Class} & $\mathbf{N_{num}}$ & $\mathbf{N_{cat}}$ & \textbf{Train} & \textbf{Val} & \textbf{Test} \\
\midrule
California Housing~\cite{California_Housing_Classification} & 2 & 8 & 0 & 13,209 & 3,303 & 4,128 \\
KC1~\cite{kc1} & 2 & 21 & 0 & 1,349 & 338 & 422 \\
Churn~\cite{churn_openml_dataset} & 2 & 16 & 4 & 3,200 & 800 & 1,000 \\
Employee~\cite{employee_dataset} & 2 & 3 & 5 & 2,977 & 745 & 931 \\
Japanese Vowels~\cite{japanese_vowels_128}  & 9 & 14 & 0 & 6374 & 1594 & 1993 \\
Internet Firewall~\cite{internet_firewall_data_542} & 4 & 7 & 0 & 41,940 & 10,485 & 13,107 \\
Dry Bean~\cite{dry_bean_602} & 7 & 16 & 0 & 8,710 & 2,178 & 2,723 \\
Estimation of Obesity Levels\cite{estimation_of_obesity_levels}  & 7 & 8 & 8 & 1350 & 338 & 423 \\
\bottomrule
\end{tabular}%

\vspace{-4mm}
\end{table}

\section{Computation cost}
\label{appendix:generation_time}

Table \ref{tab:dataset_stats} provides a detailed summary of the statistics for the eight datasets used in our experiments, including feature types and class distributions. This comprehensive overview details the partitioning of training, validation, and test sets to ensure the reproducibility of our experimental setup.

\begin{table}[th]
\centering
\caption{Computational Efficiency of Unlearnable Example Generation: We report the generation time (Seconds) across eight datasets and four model architectures (TabNet, FTT, SAINT, and MLP) to evaluate the time cost of our method.}
\label{tab:generation_time}
\begin{tabular}{l c c c c}
\toprule
Dataset & TabNet & FTT & SAINT & MLP \\
\midrule

California Housing~\cite{California_Housing_Classification} & 826.17 & 470.50 & 1461.80 & 81.45 \\
KC1~\cite{kc1} & 96.82 & 52.30 & 189.56 & 8.38 \\
Churn~\cite{churn_openml_dataset} & 209.20 & 120.68 & 421.29 & 21.44 \\
Employee~\cite{employee_dataset} & 212.34 & 108.67 & 311.25 & 18.90 \\
Japanese Vowels~\cite{japanese_vowels_128} & 390.59 & 212.05 & 763.81 & 39.02 \\
Internet Firewall~\cite{internet_firewall_data_542} & 2645.81 & 1478.72 & 4542.02 & 241.35 \\
Dry Bean~\cite{dry_bean_602} & 547.44 & 314.99 & 1111.85 & 51.37 \\
Estimation of Obesity Levels~\cite{estimation_of_obesity_levels} & 96.98 & 55.87 & 162.84 & 9.14 \\

\bottomrule
\end{tabular}

\end{table}

\section{Feature Dependency}
\label{appendix:Feature_Dependency}
To investigate the feature dependency of the trained models, we conduct a dual-aspect analysis. The left panel displays the Gradient Saliency, which quantifies the sensitivity of the model's predictions to individual numerical features by calculating the mean absolute gradients. The right panel presents the results of a Feature Ablation Study, comparing the impact of removing features in descending (Top-K) versus ascending (Bottom-K) order of importance. The sharp decline in accuracy during Top-K ablation, contrasted with the relative stability during Bottom-K ablation, validates that the model relies heavily on a specific subset of influential features.

\begin{figure*}[ht]
    \centering
    \includegraphics[width=0.75\textwidth]{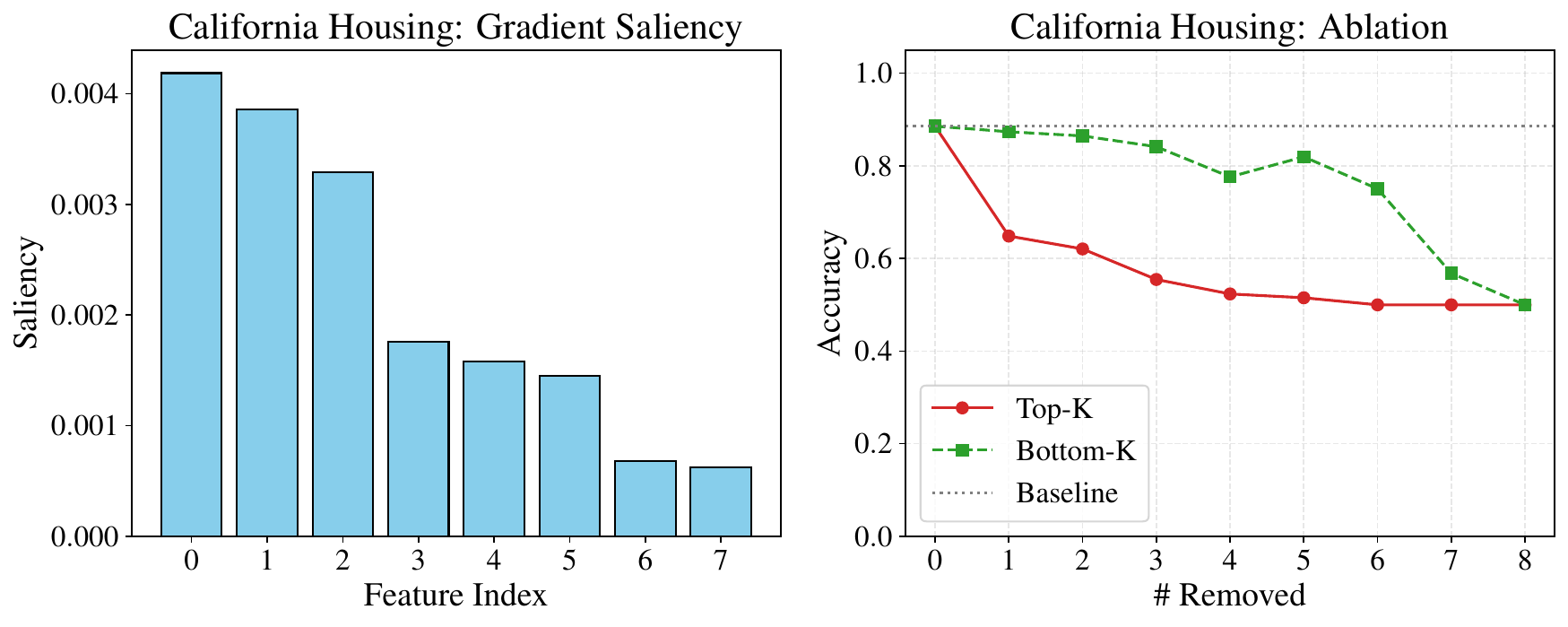}
    \vspace{-3mm}\caption{Feature gradient saliency analysis and accuracy when remove top-K and bottom-K features on California Housing dataset using FTT.}
    \label{fig:feature_ch}
    \vspace{-2mm}
\end{figure*}

\begin{figure*}[ht]
    \centering
    \includegraphics[width=0.75\textwidth]{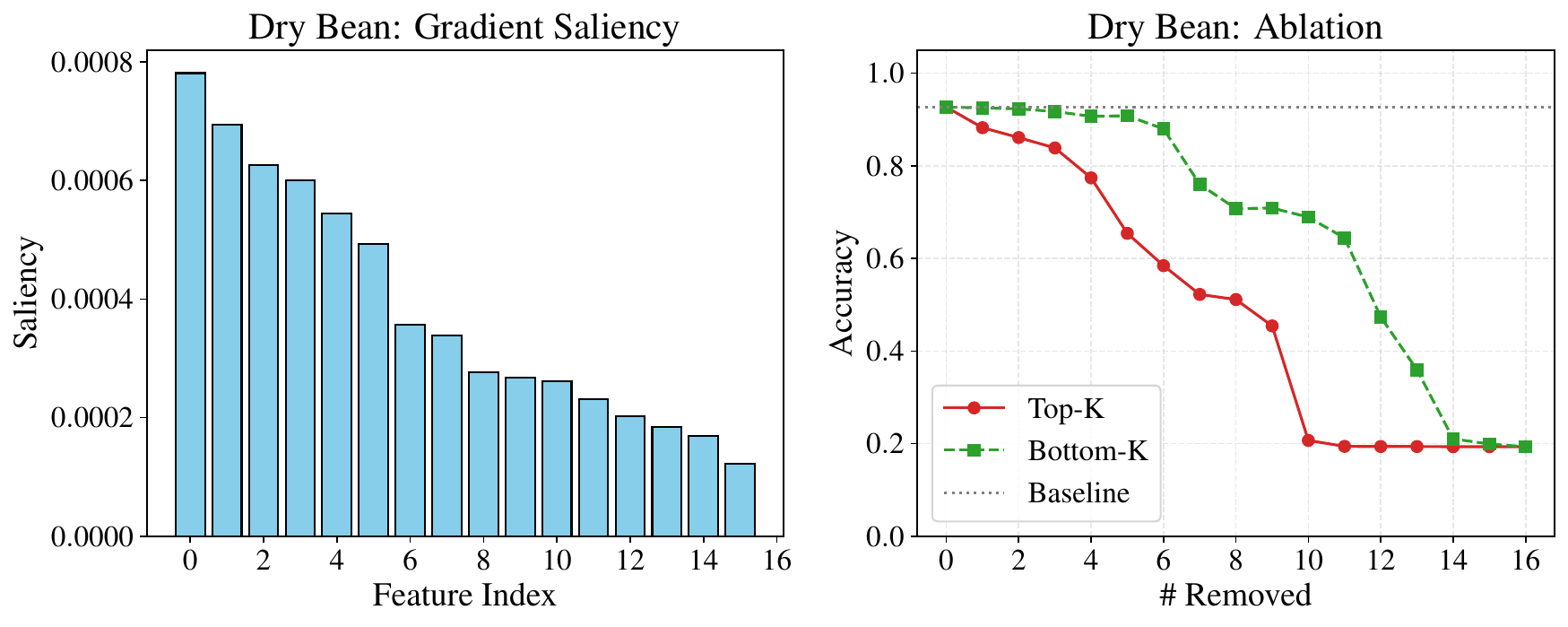}
    \vspace{-3mm}\caption{Feature gradient saliency analysis and accuracy when remove top-K and bottom-K features on Dry Bean dataset using FTT.}
    \label{fig:feature_db}
    \vspace{-2mm}
\end{figure*}

\begin{figure*}[ht]
    \centering
    \includegraphics[width=0.75\textwidth]{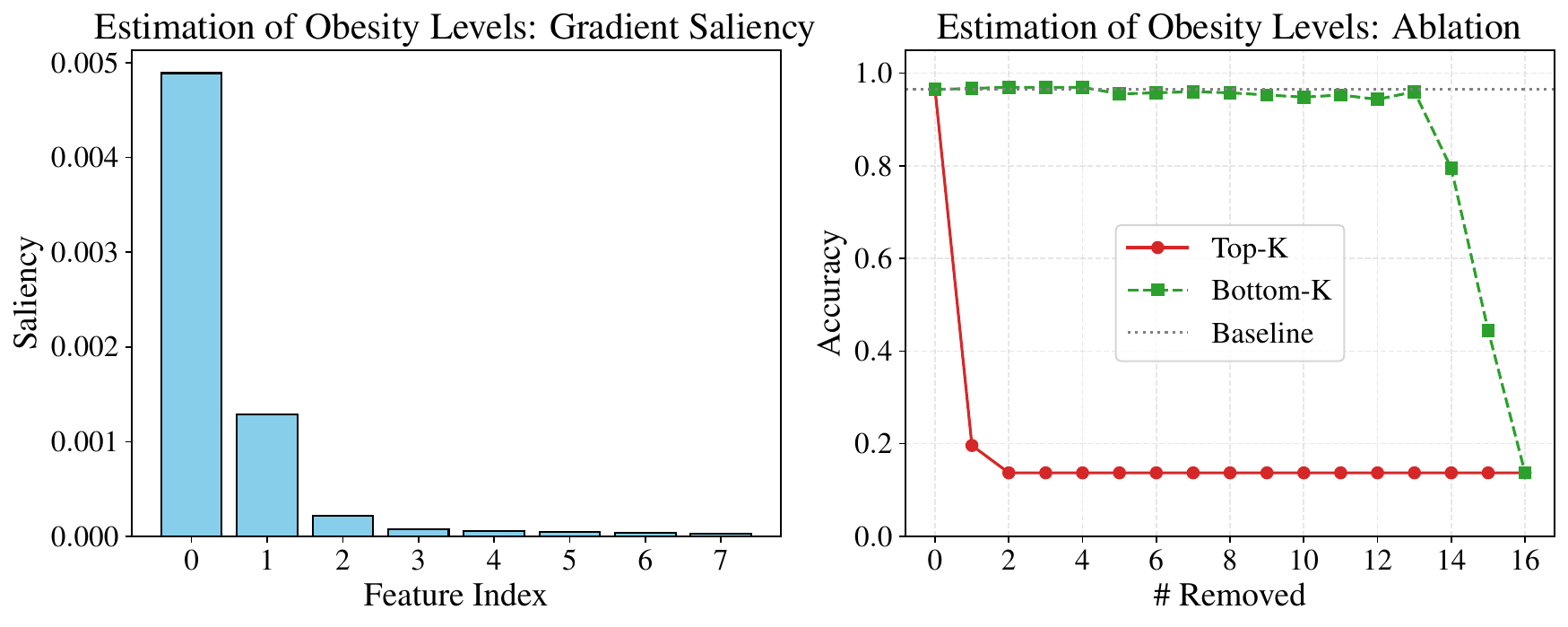}
    \vspace{-3mm}\caption{Feature gradient saliency analysis and accuracy when remove top-K and bottom-K features on Estimation of Obesity Levels dataset using FTT.}
    \label{fig:feature_eol}
    \vspace{-2mm}
\end{figure*}

\begin{figure*}[h]
    \centering
    \includegraphics[width=0.75\textwidth]{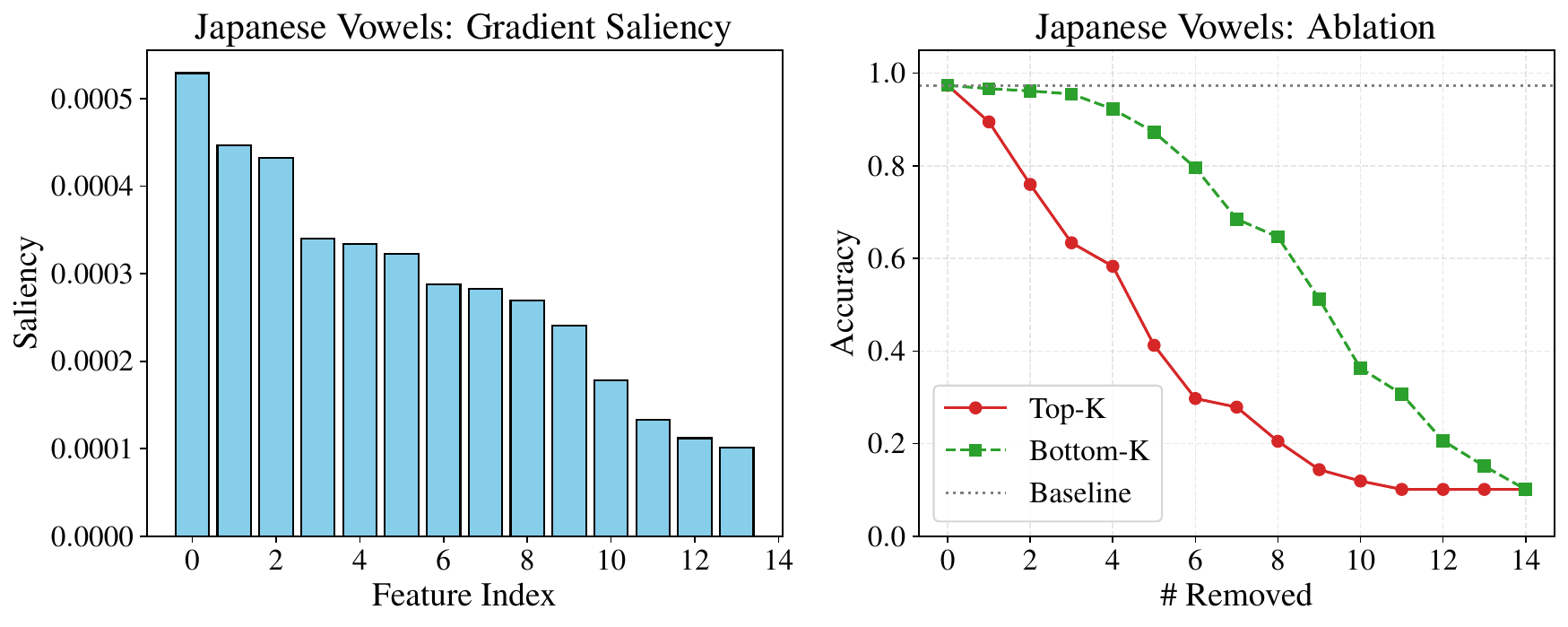}
    \vspace{-3mm}\caption{Feature gradient saliency analysis and accuracy when remove top-K and bottom-K features on Japanese Vowels dataset using FTT.}
    \label{fig:feature_jv}
    \vspace{-2mm}
\end{figure*}

\section{Limitations}
\begin{itemize}
    \item Due to page limits, we report the most representative settings. We use a single, consistent set of default hyperparameters across datasets. Light per-dataset tuning might yield slightly better numbers but is not essential to our conclusions.
    \item We use a \emph{binary} pair of modulation masks to enforce structural orthogonality (i.e., disjoint supports), which is a deliberate simplification for stability. A soft mask parameterization could be explored for extra flexibility.
\end{itemize}

\end{document}